\documentclass[twoside,11pt]{article} 
\usepackage{jmlr2e}

\usepackage{amsmath}
\usepackage{dsfont}
\usepackage{color}
\usepackage{esint}
\usepackage{enumerate}
\usepackage{amsfonts}
\usepackage{amssymb}%
\usepackage{braket}
\usepackage{bm}
\usepackage{graphicx}
\usepackage{caption}
\usepackage{subcaption}
\usepackage{longtable}
\usepackage{mathrsfs}
\usepackage{xcolor}
\usepackage{amsfonts}
\usepackage{algorithm}
\usepackage[noend]{algpseudocode}
\usepackage{float}

\makeatletter
\def\BState{\State\hskip-\ALG@thistlm}
\makeatother

\usepackage{lastpage}
\jmlrheading{20}{2019}{1-\pageref{LastPage}}{11/17; Revised
9/19}{12/19}{17-698}{Nicol\'as Garc\'ia Trillos, Zachary Kaplan, Thabo Samakhoana, Daniel Sanz-Alonso}
\ShortHeadings{On the consistency of graph-based Bayesian  semi-supervised learning}{Garc\'ia Trillos, Kaplan, Samakhoana, Sanz-Alonso}
\firstpageno{1}

\graphicspath{ {./images/} }

\providecommand{\U}[1]{\protect\rule{.1in}{.1in}}

\newtheorem{assumption}{Assumption}

\newcommand{\veps}{\varepsilon}

\renewcommand{\H}{\mathcal{H}}
\newcommand{\I}{\mathcal{I}}

\newcommand{\x}{\textbf{x}}

\newcommand{\vol}{\text{vol}_{\mathcal{M}}}
\newcommand{\M}{\mathcal{M}}

\newcommand{\length}{length}
\newcommand{\red}{\color{red}}
\newcommand{\blue}{\color{blue}}
\newcommand{\grn}{\color{mygreen}}
\newcommand{\nc}{\normalcolor}
\newcommand{\E}{\mathbb{E}}

\newcommand{\Osc}{osc}
\newcommand{\argmin}{\mbox{argmin}}
\newcommand{\1}{\mathds{1}}
\newcommand{\R}{\mathds{R}}
\newcommand{\N}{\mathds{N}}
\newcommand{\B}{\mathcal{B}}

\newcommand{\Z}{\mathds{Z}}
\renewcommand{\S}{\mathcal{S}}
\newcommand{\C}{\mathcal{C}}
\newcommand{\G}{\mathcal{G}}

\renewcommand{\O}{\mathcal{O}}
\newcommand{\F}{\mathcal{F}}
\newcommand{\h}{\mathcal{H}}
\newcommand{\X}{\mathcal{X}}
\newcommand{\T}{\mathbb{T}}
\renewcommand{\P}{\mathcal{P}}
\newcommand{\converges}[1]{ \overset{#1}{\longrightarrow}}

\newcommand{\m}{\mathbf{m}}
\newcommand{\cut}{Cut}

\newcommand{\y}{\mathbf{y}}
\newcommand{\z}{\mathbf{z}}
\newcommand{\bH}{\mathbf{H}}
\renewcommand{\H}{\mathcal{H}}
\newcommand{\bnu}{\bm{\nu}}
\newcommand{\dom}{{[0,1]^d}}
\newcommand{\spn}{Span}
\newcommand{\tacka}{\,\cdot\,}
\newcommand{\tG}{\widetilde{G}}

\newcommand{\pii}{\boldsymbol{\pi}}
\newcommand{\pit}{\widetilde{\boldsymbol{{\pi}_{n}}}}
\newcommand{\mut}{\widetilde{\boldsymbol{{\mu}_{n}}}}
\newcommand{\pin}{\boldsymbol{{\pi}_{n}}}
\newcommand{\mun}{\boldsymbol{{\mu}_{n}}}
\newcommand{\muu}{\boldsymbol{\mu}}
\newcommand{\dkl}{D_{\mbox {\tiny{\rm KL}}}}
\DeclareMathOperator{\Dom}{Dom}
\DeclareMathOperator{\Cut}{Cut}
\DeclareMathOperator{\card}{card}

\DeclareMathOperator{\trace}{Tr}
\DeclareMathOperator{\Var}{Var}
\DeclareMathOperator{\inte}{int}
\DeclareMathOperator{\tangent}{tan}
\DeclareMathOperator{\id}{Id}
\DeclareMathOperator{\supp}{supp}
\DeclareMathOperator{\grad}{grad}
\DeclareMathOperator{\dist}{dist}
\DeclareMathOperator{\relint}{resint}
\DeclareMathOperator{\diam}{diam}
\DeclareMathOperator{\diamP}{\diam_\theta(\mathcal{P}(\mathcal{K}_N))}
\DeclareMathOperator{\Ric}{Ric}
\DeclareMathOperator{\rank}{rank}
\DeclareMathOperator{\Null}{null}
\DeclareMathOperator{\Bin}{Bin}
\DeclareMathOperator{\divergence}{div}
\DeclareMathOperator{\mean}{m_2}
\DeclareMathOperator{\Proj}{Proj}
\DeclareMathOperator{\Prob}{\mathbb{P}}
\DeclareMathOperator{\esssup}{esssup}
\DeclareMathOperator{\Lip}{Lip}
\DeclareMathOperator{\Span}{span}
\newcommand{\charf}{\mathds{1}}

\definecolor{mygreen}{rgb}{0.1,0.75,0.2}

\makeatletter
\newcommand{\distas}[1]{\mathbin{\overset{#1}{\kern\z@\sim}}}%
\newsavebox{\mybox}\newsavebox{\mysim}
\newcommand{\distras}[1]{%
  \savebox{\mybox}{\hbox{\kern3pt$\scriptstyle#1$\kern3pt}}%
  \savebox{\mysim}{\hbox{$\sim$}}%
  \mathbin{\overset{#1}{\kern\z@\resizebox{\wd\mybox}{\ht\mysim}{$\sim$}}}%
}
\makeatother

\newcommand{\applied}[2]{\langle #1,#2\rangle}

\begin{document}
\title{On the consistency of graph-based Bayesian  semi-supervised learning and the scalability of sampling algorithms}

\author{\name Nicol\'{a}s Garc\'{i}a Trillos \email garciatrillo@wisc.edu \\
        \addr Department of Statistics \\
        University of Wisconsin-Madison \\
        Madison, WI 53706, USA
        \AND
        \name Zachary Kaplan  \email zachary.abraham.kaplan@gmail.com \\
	\addr Division of Applied Mathematics \\
        Brown University \\
        Providence, RI 02912, USA
        \AND
        \name Thabo Samakhoana  \email thabo\_samakhoana@alumni.brown.edu \\
        \addr Division of Applied Mathematics \\
        Brown University \\
        Providence, RI 02912, USA
        \AND
        \name Daniel Sanz-Alonso  \email sanzalonso@uchicago.edu \\
        \addr Department of Statistics \\
        University of Chicago \\
        Chicago, IL 60637, USA}

\editor{Sanjoy Dasgupta}

\maketitle

\begin{abstract}
This paper considers a Bayesian approach to graph-based semi-supervised learning. 
We show that if the graph parameters are suitably scaled, the graph-posteriors converge to a continuum limit as the size of the unlabeled data set grows. This consistency result has profound algorithmic implications:  we prove that when consistency holds, carefully designed Markov chain Monte Carlo algorithms have a uniform spectral gap, independent of the number of unlabeled inputs. Numerical experiments illustrate and complement the theory.
\end{abstract}


\makeatletter
\def\blfootnote{\gdef\@thefnmark{$\dagger$}\@footnotetext}
\makeatother

\blfootnote{All authors contributed equally to this work.}

\section{Introduction}\label{sec:introduction}

The aim of this paper is to contribute to the theoretical and methodological understanding of graph-based semi-supervised learning and its Bayesian formulation. Semi-supervised learning makes use of \emph{labeled} and \emph{unlabeled} data for training.  Labeled data consists of pairs of inputs and outputs, while unlabeled data consists only of inputs. We focus on the inductive learning task of inferring the hidden map from inputs to outputs. We work under the classical assumption that the inputs are concentrated on a low dimensional manifold embedded in a higher dimensional ambient space. Traditional graph-based optimization methods find a suitable input/output map by minimizing an objective functional comprising of at least two terms:
\begin{enumerate}[i)]
\item A regularization term involving a graph-Laplacian built using only the input data. Regularization promotes smoothness of the recovered map along the input manifold. 
\item A data-misfit term that promotes that the recovered map is accurate over the labeled data.
\end{enumerate}
Graph-based optimization methods will be reviewed below. In this paper we study 
a graph-based \emph{Bayesian} approach that, instead of recovering a single input/output map, gives a  \emph{posterior} probability distribution over maps. The posterior contains information on the most likely maps to have produced the training data, but also on the uncertainty remaining in the recovery.
As in optimization methods, the Bayesian posterior is found by balancing a smoothness penalty and a data misfit penalty. These competing forces are encoded in a prior distribution and a likelihood function.

\begin{enumerate}[I)]
\item The prior distribution serves as a regularization that promotes maps that satisfy certain smoothness conditions. The prior covariance will be defined using a graph-Laplacian built using only the input data. 
\item The likelihood function plays the role of a data-misfit functional and promotes maps that are accurate on the labeled data.
\end{enumerate}

We investigate the convergence of posterior distributions and the scaling of sampling algorithms in the limit of training large numbers of unlabeled examples. We consider $\varepsilon$-graphs, which connect any two inputs whose distance is less than $\varepsilon$. Our results guarantee that, provided that the connectivity parameter $\varepsilon$ is suitably scaled with the number of inputs, the graph-based posteriors converge, as the size of the unlabeled data set grows, to a continuum posterior. Moreover we show that, under the existence of a continuum limit, carefully designed graph-based Markov chain Monte Carlo (MCMC) sampling algorithms have a uniform spectral gap, independent of the number of unlabeled examples.   Roughly speaking our results imply that the number of Markov chain iterations needed to achieve a given accuracy is independent of the number of unlabeled data points. However, the cost per iteration will, in general, depend on the size of the data-set.

The continuum limit theory that we bring forward is of interest in three distinct ways. First, it establishes the statistical consistency of graph-based semi-supervised learning methods in machine learning; second, it suggests suitable scalings of graph parameters of practical interest (e.g. see the conditions in the parameter $s$ in Theorem \ref{th:interpolant}, the scalings for the graph connectivity $\veps$ also in Theorem \ref{th:interpolant}, and the truncation point for the spectrum of the graph-Laplacian in \eqref{def:graphprior} used to construct the prior in Theorem \ref{th:interpolant}); and third, statistical consistency is shown to go hand in hand with algorithmic scalability: when graph-based learning problems have a continuum limit, algorithms that exploit this limit structure converge in a number of iterations that is independent of the size of the unlabeled data set. The theoretical understanding of these questions relies heavily on recently developed bounds for the asymptotic behavior of the spectra of graph-Laplacians. 

Our presentation brings together various approaches to semi-supervised learning, and highlights the similarities and differences between optimization and Bayesian formulations. We include a computational study that suggests directions for further theoretical developments, and illustrates the non-asymptotic relevance of our asymptotic results.

\subsection{Problem Description} We now provide a brief intuitive problem description; a fully rigorous account is given in section \ref{sec:contandgraphBIP}. We highlight the generality of our setting, which covers a wide class of methods for semi-supervised regression and classification, including probit and logistic Bayesian methods.

We assume to be given $n$ inputs lying on an \emph{unknown} $m$-dimensional manifold $\M\subset \R^d$, $p$ of which are labeled. The collection of input data will be denoted by $\M_n = \{\x_1, \ldots, \x_n\}\subset \M,$ and we denote by $y\in \R^p$ the vector of labels. The pairs of inputs/outputs $(\x_1,y_1), \ldots, (\x_p,y_p)$ form the labeled data and the inputs $\x_{p+1},\ldots, \x_n$ are unlabeled examples. Our goal is to use the observed data to learn a label for each point in the input space (assumed to be the unknown manifold $\M$). 

In the ideal case of known manifold $\M$, a standard Bayesian approach to such learning task proceeds by putting a Gaussian process prior $\pii = N(0, -\Delta_\M^s)$ over mappings $u: \M \rightarrow \R$ and proposing a statistical model (e.g. additive Gaussian noise, probit, logistic) for the data which is encoded in a negative log-likelihood $\Phi$ .  The data model may depend on a forward map $\mathcal{F}$ that first transforms the input/output function $u$, and on the subsequent application of an observation map $\mathcal{O}$; see section \ref{sec:ForwardObservation} for concrete choices of forward and observation maps considered in this paper.  In the above, $\Delta_\M$ denotes the Laplace Beltrami operator on $\M$ and the parameter $s>0$ determines the regularity of prior draws; more intuition on the role of the Laplace Beltrami operator and the parameter $s$ will be given below.  Combining the prior and the likelihood via Bayes' rule, one can define a posterior distribution $\muu$ over functions $u:\M \to \R$ by
\begin{equation}\label{eq:posteriorknownm}
\muu(du) \propto \exp\bigl(-\Phi(u;y)\bigr) \, \pii(du).
\end{equation}
That is, the posterior is the distribution whose density with respect to the prior is proportional to the likelihoood function. 

However, as the input space $\M$ is assumed to be unknown, the above Bayesian formulation is impractical. We follow instead an \emph{intrinsic} approach and aim first at finding suitable labels for the inputs in the given point cloud $\M_n = \{\x_1, \ldots, \x_n\}$, which are then extrapolated,   via a Voronoi extension (or 1-NN extension),  to assign a label to every point on the manifold $\M$ or in the ambient space $\R^d$.  We take a Bayesian approach to learn a discrete input/output function $u_n:\M_n\to \R$ by first building a graph Laplacian which induces a Gaussian prior distribution $\pin = N(0, \Delta_{\M_n}^{-s})$ over discrete functions $u_n$, and then introducing an approximatoin $\Phi_n$ to the negative log-likelihood function $\Phi$. In this way, geometric properties of the underlying manifold $\M$  are extracted from the point cloud $\M_n$ and incorporated both in the prior and the likelihood.  Notice that if the original data model is defined in terms of some forward and observation maps, then one should also construct appropriate graph approximations for them (see section \ref{ssec:graphforward} for the approximation of the forward and observation maps considered in section \ref{sec:ForwardObservation}). \nc The solution of the graph-based Bayesian approach is a posterior distribution over discrete functions
\begin{equation}\label{eq:posteriorunknownm}
\mun(du_n) \propto \exp\bigl(-\Phi_n(u_n;y)\bigr) \, \pin(du_n).
\end{equation}
The details on how we construct ---without use of the ambient space or $\M$--- the graph-based prior $\pin$ and likelihood  $\Phi_n$ are given in section \ref{sec:contandgraphBIP}.

Two interpretations of equations \eqref{eq:posteriorknownm} and \eqref{eq:posteriorunknownm} will be useful. The first one is to see  \eqref{eq:posteriorunknownm} as a graph-based discretization of a Bayesian inverse problem over functions on $\M$ whose posterior solution is given by equation \eqref{eq:posteriorknownm}. The second is to interpret them as classical Bayesian regression problems. In the latter interpretation, $\M$ may represent a low-dimensional manifold  sufficient to  characterize  features living in an extremely high dimensional ambient space ($m\ll d$), perhaps upon some dimensionality reduction of the given inputs; in the former, $\M$ may represent the unknown physical domain of a differential equation. We note again that our framework covers ---by the flexibility in the choice of misfit functional $\Phi$--- a wide class of classification and regression learning problems that includes Bayesian probit and logistic models.

Our first goal is to study the large $n$ limit of the posterior distribution $\mun$ after it has been pushed-forward by the interpolation map $\mathcal{I}^1_n$ (see definition \eqref{def:I}) that extends functions defined on $\M_n$ to functions defined on $\M$. Our second goal is to study the algorithmic scalability of carefully designed MCMC schemes to sample from $\mun$ (see Algorithm \ref{pCN-RRW. }).   The theory on statistical consistency and algorithmic scalability that we set forth concerns regimes with large number $n$ of input training data and moderate number $p$ of labeled examples. This is precisely the regime of interest in semi-supervised learning applications, where often labeled data is expensive to collect but unlabeled data abounds. Our consistency results guarantee that graph-based posteriors of the form \eqref{eq:posteriorunknownm} are close to a ground truth posterior of the form \eqref{eq:posteriorknownm}, while the algorithmic scalability that we establish ensures the convergence, in an $n$-independent number of iterations, of certain MCMC methods for graph posterior sampling. The computational cost per iteration may, however, grow with $n$.
These MCMC methods are in principle applicable in fully supervised learning, but their performance would typically deteriorate if both $n$ and $p$ are allowed to grow.
Finally, we note that although our exposition is focused on semi-supervised regression, our conclusions are equally relevant for semi-supervised classification.

\subsection{Literature}\label{ssec:literature}
Here we put into perspective our framework by contrasting it with optimization and extrinsic approaches to semi-supervised learning, and by relating it to other surrogate and approximate methods for  Bayesian inversion. We also give some background on MCMC algorithms.

\subsubsection{Graph-Based Semi-supervised Learning}
We refer to  \cite{zhu2005semi} for an introductory tutorial on semi-supervised learning with useful pointers to the literature. The question of when and how unlabeled data matters is addressed in \cite{sayan}.  Some key papers on graph-based methods are \cite{zhu2003semi,hartog2016nonparametric,blum2001learning}. 

As already noted, a key motivation for graph-based semi-supervised learning is that high dimensional inputs can often be represented in a low-dimensional manifold, whose local geometry may be learned by imposing a graph structure on the inputs. In practice, features may be close to but not exactly \emph{on} an underlying manifold \citep{ruiyilocalregularization}. The question of how to find suitable manifold representations has led to a vast literature on dimensionality reduction techniques and manifold learning, e.g. \cite{roweis2000nonlinear,tenenbaum2000global,donoho2003hessian,belkin2004semi}.

 The reconstruction of the hidden input/output maps from few labeled examples can be carried out by compromising between data fidelity and regularization (along the underlying manifold). Our work considers regularizations defined in terms of the graph Laplacian $\Delta_{\M_n}^{-s},$ with the power parameter $s>0$ tuning the amount of regularization  (the higher $s$ the more regularity imposed).  Although the use of such parameter is standard in the machine  learning  literature  \citep{sindhwani2005beyond} our work provides new understanding on how $s$ should be chosen in terms of the intrinsic dimension $m$ of the input manifold in order to have consistent learning in the limit of large numbers of unlabeled examples.  Our analysis builds on recent results from \cite{SpecRatesTrillos} where explicit rates of convergence for the spectra of graph Laplacians towards the spectrum of a continuum differential operator have been obtained. These results relate in a quantitative way the geometry of the underlying manifold $\M$ and that of the point cloud $\M_n$. The problem of studying the large sample limit of graph Laplacians has received much attention in the last decades. Initially, most results were of pointwise type as in \cite{HeAuvL07,bel_niy_LB,GK,Hei2006,singer06,THJ}. More recently, the focus has been given to variational and spectral convergence \cite{belkin2007convergence,SinWu13,trillosACHA,Shi2015,burago2014graph,SpecRatesTrillos,ruiyilocalregularization}.

Alternative graph $\mathtt p$-Laplacian regularizations were introduced in \cite{zhou2005regularization}.  This type of regularization is similar to the one considered in this paper, but it does not induce a Gaussian prior on the hidden input/output map; because of this, it is more difficult to implement algorithms to sample from posteriors based on $\mathtt p$-Laplacian regularization.  The statistical consistency of semi-supervised learning based on $\mathtt p$-Laplacian regularization has been studied in \cite{el2016asymptotic}, \cite{slepvcev2017analysis}. These papers have rigorously analyzed how the parameter $\mathtt  p$ ---which plays an analogous role to $s$ in our context--- should be chosen in terms of dimension so that ``labels are not forgotten" in the large data limit.


\subsubsection{Bayesian vs.\ Optimization, and Intrinsic vs.\ Extrinsic}
In this subsection we focus on the regression interpretation, with labels directly obtained from noisy observation of the unknown input/output function. 
The Bayesian formulation that we consider has the advantage over traditional optimization formulations in that it allows for uncertainty quantification in the recovery of the unknown function \cite{bertozziuncertainty}. Moreover, from a computational viewpoint, we shall show that certain sampling algorithms have desirable scaling properties ---these algorithms, in the form of simulated annealing, may also find application within optimization formulations \citep{geyer1995annealing}.

The Bayesian update \eqref{eq:posteriorunknownm} is intimately related to the optimization problem
\begin{equation}
\min_{u_n } \,\, \langle\Delta_{\M_n}^{s} u_n ,  u_n \rangle + \Phi_n(u_n; y).
\label{min:SA}
\end{equation}
Here $\Delta_{\M_n}$ represents the graph-Laplacian, as defined in equation \eqref{eq:graphlaplacian} below, and the minimum is taken over square integrable functions on the point cloud $\M_n.$ 
Precisely, the solution $u_n^*$ to \eqref{min:SA} is the mode (or MAP for \emph{maximum a posteriori}) of the posterior distribution $\mun$ in  \eqref{eq:posteriorunknownm} with a Gaussian prior $\pin = N(0,  \Delta_{\M_n}^{-s})$. 

The Bayesian problem \eqref{eq:posteriorunknownm} and the variational problem \eqref{min:SA} are \emph{intrinsic} in the sense that  they are constructed without reference to the ambient space (other than through its metric), working in the point cloud $\M_n.$ In order to address the \emph{generalization problem} of assigning labels to points $\x \notin \M_n$ we use interpolation maps that turn functions defined on the point cloud into functions defined on the ambient space. We will restrict our attention to the family of $k$-NN interpolation maps defined by
\begin{equation}
\bigl[\mathcal{I}^k_n(u_n)\bigr] (\x) := \frac{1}{k}\sum_{\x_i \in N_k(\x) } u_n(\x_i), \quad \x \in \R^d, 
\label{def:I}
\end{equation}
where $N_k(\x)$ is the set of $k$-nearest neighbors in $\M_n$ to $\x$;  here the distance used to define nearest neighbors is that of the ambient space. Within our Bayesian setting we consider $\mathcal{I}_{n \sharp} \mun$, the push-forward of $\mun$ by $\mathcal{I}_n$, as the fundamental object that 
allows us to assign labels to inputs $\x\notin \M_n$, and quantify the uncertainty in such inference.
The need of interpolation maps also appears in the context of intrinsic variational approaches to binary classification \citep{trillos2017new} and in the context of variational problems of the form \eqref{min:SA}: the function $u_n^*$ is only defined on $\M_n,$ and hence should be extended to the ambient space via an interpolation map $\mathcal{I}_n$.

Intrinsic approaches contrast with \emph{extrinsic} ones, such as manifold regularization \citep{belkin2005towards,belkin2006manifold}. This method solves a variational problem of the form
\begin{equation}
\label{min:Belkin}
\min_{u } \,\,  \langle \Delta_{\M_n}^{s} u|_{\M_n} , u|_{\M_n} \rangle  +  \Phi(u;y) + \zeta \lVert u \rVert^2_{\mathcal{H}_K},
\end{equation}
where now the minimum is taken over functions in a reproducing kernel Hilbert space $\H_K$ defined over the \textit{ambient} space $\R^d,$ and $u|_{\M_n}$ denotes the restriction of $u$ to $\M_n.$ The kernel $K$ is defined in $\R^d$ and the last term in the objective functional, not present in \eqref{min:SA}, serves as a regularizer in the ambient space; the parameter $\zeta \geq 0$ controls the weight given to this new term. Bayesian and extrinsic formulations may be combined in future work.

In short, \emph{extrinsic} variational approaches solve a problem of the form \eqref{min:Belkin}, and \emph{intrinsic} ones solve \eqref{min:SA} and then generalize by using an appropriate interpolation map. In the spirit of the latter, the \emph{intrinsic} Bayesian approach of this paper defines an intrinsic graph-posterior by \eqref{eq:posteriorunknownm} and then this posterior is pushed-forward by an interpolation map. What are the advantages and disadvantages of each approach? Intuitively, the intrinsic approach seems more natural for label inference of inputs \emph{on} or \emph{close to} the underlying manifold $\M.$ However, the extrinsic approach is appealing for problems where no low-dimensional manifold structure is present in the input space.

\subsubsection{Approximate and Surrogate Bayesian Learning}
Our learning problem can be seen as approximating a ground-truth Bayesian inverse problem over functions on the underlying manifold $\M$ \citep{DS15,trillos2016bayesian,harlim2019kernel}.  Traditional problem formulations and sampling algorithms require repeated evaluation of the likelihood, often making naive implementations impractical. For this reason, there has been recent interest in reduced order models \citep{sacks1989design,kennedy2001bayesian,arridge2006approximation,cui2015data},  and in defining surrogate likelihoods in terms of Gaussian processes \citep{rasmussen2006gaussian,stein2012interpolation,stuart2017posterior}, or polynomial chaos expansions \citep{xiu2010numerical,marzouk2007stochastic}. Pseudo-marginal \citep{beaumont2003estimation} and approximate Bayesian computation methods \citep{beaumont2002approximate} have become popular in intractable problems where evaluation of the likelihood is not possible. There are two distinctive aspects of the graph-based models employed here. First, they approximate both the prior and the likelihood; and second, the approximate and ground-truth posteriors live in different spaces: the former is a measure over functions on a point cloud, while the latter is a measure over functions on the continuum. The paper \cite{garcia-sanz2017continuum} studied the continuum limits of graph-posteriors to the ground-truth continuum posterior. This was achieved by using a new topology inspired by the analysis of  functionals over functions in point clouds arising in machine learning \citep{trillos,trillos2014canadian,trillosACHA,slepvcev2017analysis}.

In this paper, we rigorously make a connection between the graph Bayesian model and the continuum one, by proving that in the large number of unlabeled data limit, the extended graph posterior converges towards the posterior of the continuum Bayesian model.

\subsubsection{Markov Chain Monte Carlo}
MCMC is a popular class of algorithms for posterior sampling.
Here we consider certain Metropolis--Hastings MCMC methods that construct a Markov chain that has the posterior as its invariant distribution by sampling from a user-chosen proposal and accepting/rejecting the samples using a general recipe. Posterior expectations are then approximated by averages with respect to the chain's empirical measure. The generality of Metropolis--Hastings algorithms is a double-edged sword: the choice of proposal may have a dramatic impact on the convergence of the chain. Even for a given form of proposal, parameter tuning is often problematic. These issues are exacerbated in learning problems over functions, as traditional algorithms often break-down.

The preconditioned Crank-Nicolson  (pCN) algorithm introduced in \cite{beskos2008mcmc}  allows for scalable sampling of infinite dimensional functions provided that the target is suitably defined as a change of measure. Indeed, the key idea of the method is to exploit this change of measure structure, that arises naturally in Bayesian nonparameterics but also in the sampling of conditioned diffusions and elsewhere. Robustness is understood in the sense that, when pCN is used to sample functions projected onto a finite $D$-dimensional space, the rate of convergence of the chain is independent of $D.$ This was already observed in \cite{beskos2008mcmc} and \cite{cotter2013mcmc}, and was further understood in \cite{hairer2014spectral} by showing that projected pCN methods have a uniform spectral gap, while traditional random walk does not. 

In this paper we substantiate the use of graph-based pCN MCMC algorithms \citep{bertozziuncertainty} in semi-supervised learning. The main insight is that our continuum limit results provide the necessary change of measure structure for the robustness of pCN. This allows us to establish their uniform spectral gap in the regime where the continuum limit holds. Namely, we show that if the number $p$ of labeled data is fixed, then the rate of convergence of graph pCN methods for sampling graph posterior distributions is independent of $n$.  We remark that pCN addresses some of the challenges arising from sampling functions, but fails to address challenges arising from tall data. Some techniques to address this complementary difficulty are reviewed in \cite{doucettalldata}.

\subsection{Paper Organization and Main Contributions}

A thorough description of our setting is given in section \ref{sec:contandgraphBIP}. Algorithms for posterior sampling are presented in section \ref{sec:sampling}.  Section \ref{sec:theory} contains our main theorems on continuum limits of graph posteriors and uniform spectral gaps. Finally, a computational study is conducted in section \ref{sec:Numerics}. All proofs and technical material are collected in an appendix.

The two main theoretical contributions of this paper are Theorem \ref{th:interpolant} ---establishing statistical consistency of intrinsic graph methods generalized by means of interpolation maps--- and Theorem \ref{th:pcn} ---establishing the uniform spectral gap for graph-based pCN methods under the conditions required for the existence of a continuum limit. 
Both results require appropriate scalings of the graph connectivity with the number of inputs. An important contribution of this paper is the analysis of truncated graph-priors that retain only the portion of the spectra of the graph Laplacian that provably approximates that of the ground-truth continuum.  As it turns out, only a portion of the spectrum of the graph Laplacian contains relevant information about the underlying manifold $\M$, and thus one can disregard higher modes. See the discussion in section \ref{ssec:approxbounds} and Figure \ref{fig:Laplacian Spectrum} for an illustration of this.   \nc

From a numerical viewpoint, our experiments illustrate parameter choices that lead to successful graph-based inversion, highlight the need for a theoretical understanding of the spectrum of graph Laplacians and of regularity of functions on graphs, and show that the asymptotic consistency and scalability analysis set forth in this paper is of practical use outside the asymptotic regime.

\section{Setting}\label{sec:contandgraphBIP}
Throughout, $\M$ will denote an $m$-dimensional, compact, smooth manifold embedded in $\R^d$.
We let  $\M_n:=\{\x_1, \dots, \x_n \}$ be a collection of i.i.d. samples from the uniform distribution on  $\M$. We are interested in learning functions defined on $\M_n$ by using the inputs $\x_i$ and some output values, obtained by noisy evaluation at $p\le n$ inputs of a transformation of the unknown function. The learning problem in the discrete space $\M_n$ is defined by means of a  graph-based discretization of a continuum learning problem defined over functions on $\M.$ We view the continuum problem as a ground-truth case where full geometric information of the input space is available. We describe the continuum learning setting in subsection \ref{ssec:continuous}, followed by the discrete learning setting in subsection \ref{ssec:discrete}. We will denote by  $L^2(\gamma)$ the space of functions on the underlying manifold that are square integrable with respect to the uniform measure $\gamma$. We use extensively that functions in $L^2(\gamma)$ can be written in terms of the (normalized) eigenfunctions $\{\psi_i\}_{i=1}^\infty$ of the Laplace Beltrami operator $\Delta_\M$. We denote by $\{\lambda_i\}_{i=1}^\infty$ the associated eigenvalues of $-\Delta_\M$, assumed to be in non-decreasing order and repeated according to multiplicity. Analogous notations will be used in the graph-based setting, with scripts $n$.

\subsection{Continuum Learning Setting}\label{ssec:continuous}
Our ground-truth continuum learning problem consists of the recovery of a function $u\in L^2(\gamma)$ from data $y\in \R^p.$ The data $y$ are assumed to be a noisy observation of a vector $v \in \R^p$ obtained indirectly from the function of interest $u$ as follows:

$$ u \in L^2(\gamma) \mapsto v :=  \O \circ \F (u) \mapsto y.$$
Here $\F: L^2(\gamma) \to L^2 (\gamma)$ is interpreted as a \emph{forward map}  representing, for instance, a map from inputs to outputs of a differential equation. As a particular case of interest, $\F$ may be the identity map in $L^2(\gamma).$ The map $\O : L^2(\gamma) \to \R^p$ is interpreted as an observation map, and is assumed to be linear and continuous. The Bayesian approach that we will now describe proceeds by specifying a prior on the unknown function $u$, and a noise model for the generation of data $y$ given the vector $v = \O \circ \F(u).$ The solution is a posterior measure $\muu$ over functions on $\M$, supported on $L^2(\gamma).$

 \subsubsection{Continuum Prior}
We assume a Gaussian {\emph{prior}} distribution $\boldsymbol{\pi}$ on the unknown initial condition $u\in L^2(\gamma)$:
\begin{equation}\label{eq:prior}
\boldsymbol{\pi} = N(0,\C_u), \quad \quad \C_u = (\alpha I - \Delta_\M)^{-s/2},
\end{equation}
where $\alpha\ge 0$, $s>m$ and $\Delta_\M$ denotes the Laplace Beltrami operator. Equation \eqref{eq:prior} corresponds to the covariance operator description of the Gaussian measure $\pii$. The covariance function representation may be advantageous in the derivation of regression formulae ---see the appendix. As described for instance in \cite{gao2019gaussian}, the Laplace Beltrami operator is a natural object to define Gaussian processes on manifolds, because its eigenfunctions contain rich geometric information. To provide further intuition, note that draws $u\sim \pii$ can be obtained via the Karhunen-Lo\`{e}ve expansion
\begin{equation}\label{eq:karhunenloeve}
u(x) = \sum_{i=1}^\infty \ (\alpha +\lambda_i )^{-s/4} \xi_i\psi_i(x), \quad \quad  \xi_i \distas{\text{i.i.d}} N(0,1),
\end{equation}
showing that the prior $\pii$ favors functions that have larger components in the first eigenfunctions of $\Delta_\M.$ The condition $s>m$ guarantees that the expected $L^2(\gamma)$ norm of $u\sim \pii,$ which agrees with $ \sum_{i=1}^\infty \ (\alpha +\lambda_i )^{-s/2},$ is finite. This in turn implies that  $u \sim \boldsymbol{\pi}$ belongs to $L^2(\gamma)$ almost surely. More generally, the parameter $s$ characterizes the almost sure H\"older and Sobolev regularity of draws from $\pii$ \citep{DS15}; larger values of $s$ correspond to smoother prior draws. The parameter $\alpha$ gives an effective prior length-scale: frequencies corresponding to $\lambda_i \ll \alpha$ have substantial contribution in the sum in equation \eqref{eq:karhunenloeve}.  

%
%
%

\subsubsection{Continuum Forward and Observation Maps}
\label{sec:ForwardObservation}
In what follows we take, for concreteness and motivated by applications in image deblurring, the forward map $\F=\F^t$ to be the solution of the heat equation on $\M$ up to a given time $t \ge 0.$ That is, we set 
\begin{equation}\label{eq:heat}
\F u \equiv \F^t u := e^{t\Delta_{\M}} u.
\end{equation}
Note that $\M$ plays two roles in definition of $\F^t$: it determines both the physical domain of the heat equation and  the Laplace Beltrami operator $\Delta_\M$.   Our choice of forward map $\mathcal{F}^t$  includes the identity map (corresponding to regression) as a particular case (for $t=0$) and gives us the opportunity to study slightly more general data models. We note that $\mathcal{F}^t$ has a natural graph counterpart (see \eqref{heatgraph}).     \nc

We now describe our choice and interpretation of observation maps. Let $\x_1, \ldots, \x_p\in \M$, and let $\delta>0$ be small. For $w \in L^2(\gamma)$ we define the $j$-th coordinate of the vector $\O w$ by 
\begin{equation}\label{eq:observationmap}
[\O w]_j :=\frac{1}{\gamma\bigl(B_{\delta}(\x_j) \cap \M\bigr)} \int_{B_\delta( \x_j) \cap \M} w(x) \gamma(dx), \quad 1\le j \le p,
\end{equation}
where $B_\delta ( \x_j)$ denotes the Euclidean ball of radius $\delta$ centered at $\x_j.$ At an intuitive level, and in our numerical investigations, we see $\O$ as the point-wise evaluation map at the inputs $\x_j$:
\begin{equation*}
\O w = [w(\x_1), \ldots, w(\x_p)]' \in \R^p.
\end{equation*}
Henceforth we denote $\G:= \O \circ \F.$

\begin{remark}
It would be perhaps more intuitive to work with an  observation map defined by pointwise evaluations rather than local averages at a certain length-scale $\delta$. Indeed, typically one assumes that the observations $y$ correspond to noisy versions of ``true" labels associated to given feature vectors. However, for technical reasons when going from discrete to continuum in the next sections, in the very low number of observed labels regime that we work on (i.e. $p$ does not grow to infinity with $n$) definition \ref{eq:observationmap} allows us to perform rigorous analysis in an $L^2$ sense, while pointwise evaluation does not. It is still an open problem to establish uniform type convergence results for eigenvectors of graph Laplacians towards continuum counterparts in the random geometric graph setting; such technical results would allow us to work with the more standard setting for the observation map.

Having said this, when the continuum prior $\boldsymbol{\pi}$ is supported on a space of regular functions (as is the case when $s$ in \eqref{eq:prior} is large enough), the posterior (as defined in \ref{def:posterior}) converges in the limit $\delta \rightarrow 0$ to a posterior obtained from a likelihood where the observation map was based on pointwise evaluations. Thus, for strong priors we do not expect much difference between working with one observation model or the other.    
\end{remark}


\subsubsection{Data and Noise Models}\label{ssec:dataandnoise}
Having specified the forward and observation maps $\F$ and $\O,$ we assume that the label vector $ y \in \R^p$ arises from noisy measurement of $\O\circ \F (u) \in \R^p.$
A noise-model will be specified via a function  $\phi^y: \R^p \to \R.$ We postpone the precise statement of assumptions on $\phi^y$ to section \ref{sec:theory}. Two guiding examples, covered by the theory, are given by
\begin{equation}\label{eq:psichoices}
\phi^y (w) := \frac{1}{2\sigma^2} |y - w |^2, \quad \text{or} \quad \phi^y (w):= - \sum_{i=1}^p \log\Bigl( \Psi\bigl(y_i w_i;\sigma\bigr)\Bigr),
\end{equation}
where $\Psi$  denotes the CDF of a centered univariate Gaussian with variance $\sigma^2.$
The former noise model corresponds to Gaussian i.i.d. noise in the observation of each of the $p$ coordinates of $\G u.$ The latter corresponds to probit classification, and a noise model of the form $y_i = S\bigl( v_i + \eta_i\bigr)$ with $\eta_i$ i.i.d. $N(0,\sigma^2),$ and $S$ the sign function. For label inference in Bayesian classification, the posterior obtained below needs to be pushed-forward via the sign function \citep{bertozziuncertainty}.
\subsubsection{Continuum Posterior}
The Bayesian solution to the ground-truth continuum learning problem is a continuum posterior measure
\begin{align}\label{def:posterior}
\begin{split}
\muu(du) \propto \exp\bigl(-\phi^y(\G u) \bigr) \pii(du) \\
=: \exp\bigl(-\Phi(u;y)\bigr) \pii(du),
\end{split}
\end{align}
that represents the conditional distribution of $u$ given the data $y.$ Equation \eqref{def:posterior} defines the negative log-likelihood function $\Phi$, that characterizes the conditional distribution of labels $y$ given $u.$ The posterior $\muu$ contains all the information on the unknown input $u$ available in the prior and the data.

\subsection{Discrete Learning Setting}\label{ssec:discrete}
We consider the learning of functions defined on a point cloud $\M_n:= \{\x_1, \dots, \x_n \}\subset \M.$ The underlying manifold $\M$ is assumed to be unknown. We suppose to have access to the same label data $y$  as in the continuous setting, and that the inputs $\x_1, \ldots, \x_p$ in the definition of $\O$ correspond to the first $p$ points in $\M_n$. Thus, in a physical analogy the data may be interpreted as noisy measurements of the true temperature at the first $p$ points in the cloud at time $t\ge 0.$  The aim is to construct ---without knowledge of $\M$---  a  posterior measure $\mun$ over functions in $\M_n$ representing the initial temperatures at each point in the cloud. 

Similar to the continuous setting, we will denote by $L^2(\gamma_n)$ the space of functions on the cloud that are square integrable with respect to the uniform measure $\gamma_n$ on $\M_n$. It will be convenient to view, formally, functions $u_n \in L^2(\gamma_n)$ as vectors in $\R^n$. We then write $u_n\equiv [u_n(1), \ldots, u_n(n)]',$ and think of $u_n(i)$ as evaluation of the function $u_n$ at $\x_i.$

The graph-posteriors are built by introducing  a graph-based prior, and  graph-based forward and observation maps $\F_n: L^2(\gamma_n) \to L^2(\gamma_n)$ and $\O_n: L^2(\gamma_n) \to \R^p$. The same noise-model and data as in the continuum case will be used.
We start by introducing a graph structure in the point cloud. Graph-based priors and forward maps are defined via a graph-Laplacian that summarizes the geometric information available in the point cloud $\M_n.$

\subsubsection{Geometric Graph and Graph-Laplacian}
We endow the point cloud with a graph structure. We focus on $\varepsilon$-neighborhood graphs: an input is connected to every input within a distance of $\varepsilon.$ A popular alternative are $k$-nearest neighbor graphs, where an input is connected to its $k$ nearest neighbors. The influence of the choice of graphs in unsupervised learning is studied in \cite{maier2009influence}. 

First, consider the kernel function $K : [0,\infty) \rightarrow [0, \infty)$ defined by
\begin{equation}
K(r):=\begin{cases}1 & \text{ if } r \leq 1, \\ 0 & \text{otherwise.} \end{cases}
\label{eta}
\end{equation}
For $\veps>0$ we let $K_\veps: [0,\infty) \rightarrow [0, \infty)$ be the rescaled version of $K$ given by
\[ K_\veps(r) := \frac{m+2}{ n^2 \alpha_m  \veps^{ m+2 }}K\left( \frac{r}{\veps} \right),\]
where $\alpha_m$ denotes the volume of the $m$-dimensional unit ball.
We then define the weight $W_{n}(\x_i, \x_j)$ between $\x_i, \x_j \in \M_n$ by
\[ W_n(\x_i,\x_j) := K_{\veps_n}(\lvert  \x_i-\x_j \rvert), \]
for a given choice of parameter $\veps=\veps_n$, where we have made the dependence of the \emph{connectivity rate} $\veps$ on $n$ explicit. In order for the graph-based learning problems to be consistent in the large $n$ limit, $\varepsilon$ should be scaled appropriately with $n$ ---see subsection \ref{ssec:continuumlimits}. Figure \ref{fig:graphs} shows three geometric graphs $(\M_n,W_n)$ with fixed $n$ and different choices of connectivity $\varepsilon.$

\begin{figure}
  \includegraphics[width=\linewidth]{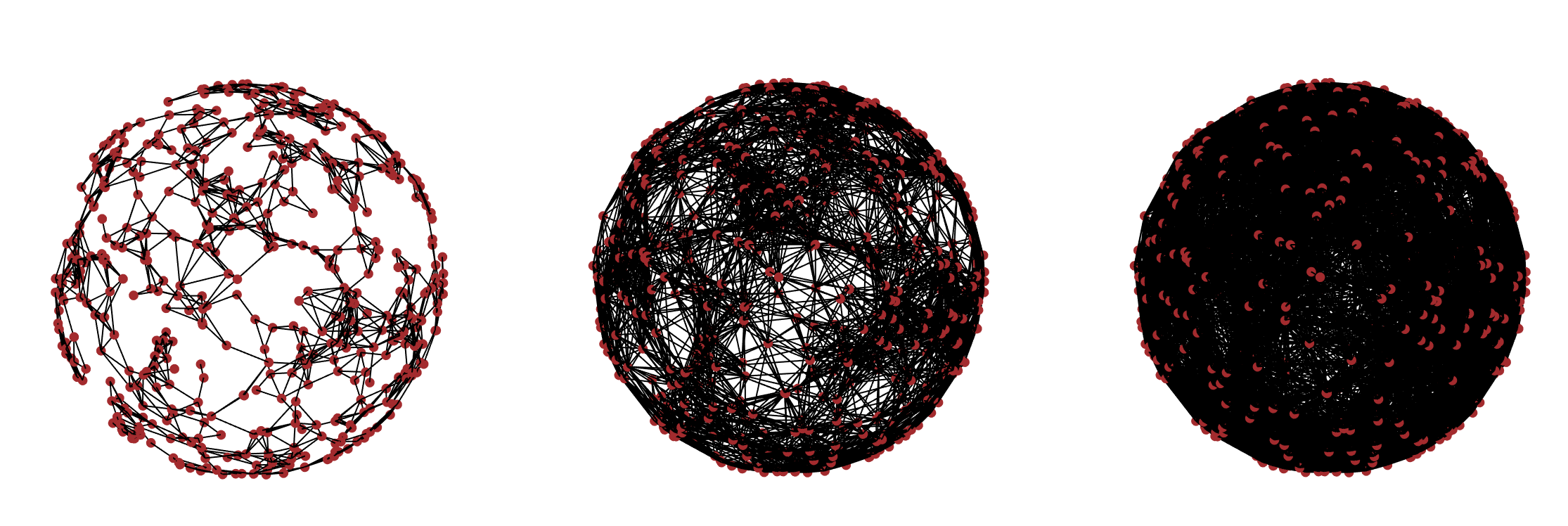}
  \caption{Geometric graphs with $n=500$, and $\varepsilon = n^{-1/4}, 2n^{-1/4},$ and $3n^{-1/4}$ from left to right.}
  \label{fig:graphs}
\end{figure}

We now define the  \textit{graph Laplacian} of the geometric graph $(\M_n, W_n)$ by
\begin{equation}\label{eq:graphlaplacian}
\Delta_{\M_n} := D_n -W_n,  
\end{equation}
where $D$ is the degree matrix of the weighted graph, i.e.,  the diagonal matrix with diagonal entries $D_{ii}= \sum_{j=1}^n W_{n}(\x_i, \x_j)$. Several definitions of graph Laplacian co-exist in the literature; the one above is some times referred to as the \textit{unnormalized} graph Laplacian \cite{von2007tutorial}.  As will be made precise, the performance of the learning methods considered here is largely determined by the behavior of the spectrum of the graph Laplacian.  Throughout we denote its eigenpairs by $\{\lambda_i^n,\psi_i^n \}_{i=1}^n$, and assume that the eigenvalues are in non-decreasing order.

\subsubsection{Graph Prior}
A straight-forward discrete analogue to \eqref{eq:prior} suggests endowing the unknown function $u_n$ with a prior
\begin{equation}\label{priordiscrete}
\pit = N(0, \C_{u_n}), \quad \quad \C_{u_n}:=(\alpha I_n +  \Delta_{\M_n})^{-s/2},
\end{equation}
where $\alpha \ge 0$ and $s>m$ are chosen as in \eqref{eq:prior}. Like the continuum prior, the graph-based one favors functions $u_n$ with large components in the first eigenfunctions of $\Delta_{\M_n},$ thus infusing geometric information on the probabilistic Bayesian reconstruction \citep{bertozziuncertainty}.
The graph Laplacian, in contrast to the regular Laplacian, is positive semi-definite, and hence the change in sign with respect to \eqref{eq:prior}. This choice of graph prior was considered in \cite{garcia-sanz2017continuum}, and also in \cite{bertozziuncertainty} in the case $\alpha=0, s =2$. In this paper we introduce and study priors $\pin$ defined in terms of truncation of the priors $\pit$, retaining only the portion of the spectra of $\Delta_{\M_n}$ that provably approximates that of $-\Delta_{\M}.$

Precisely, we define the graph priors $\pin$ as the law of $u_n$ given by
\begin{equation}\label{def:graphprior}
u_n = \sum_{i=1}^{k_n}   (\alpha +\lambda_i ^n)^{-s/4}  \xi_i  \psi_i^n, \quad \quad  \xi_i \distas{\text{i.i.d}} N(0,1),
\end{equation}
where $k_n\le n$ may be chosen freely with the restrictions that $k_n\rightarrow \infty$ and $\lim_{n\to \infty} {k_n} \varepsilon_n^m = 0.$  Such choice is possible as long as the connectivity $\varepsilon_n$ decays with $n$.

\subsubsection{Graph Forward and Observation Maps}\label{ssec:graphforward}
We define a forward map $\F_n: L^2(\gamma_n) \to L^2(\gamma_n)$ by 
\begin{equation}\label{heatgraph}
\F_n u_n \equiv \F_n^t u_n := e^{-t\Delta_{\M_n}} u_n,
\end{equation}
where $t\ge 0$ is given as in the continuum setting.
Likewise, for $\delta>0$ as in \eqref{eq:observationmap} we define an observation map $\O_n: L^2(\gamma_n) \to \R^p$ by
$$[\O_n w](j)  :=\frac{1}{n \gamma_n\bigl(B_{\delta}(\x_j) \bigr)} \sum_{k: \,\x_k\in B_\delta( \x_j) \cap \M_n} w(k), \quad 1\le j \le p.$$ 
As in the continuum setting, $\O_n$ should be thought of as point-wise evaluation at the inputs $\{\x_i \}_{i=1}^p$ and we denote $\G_n :=\O_n \circ \F_n.$

\subsubsection{Data and Likelihood}\label{ssec:graphdata} 
For the construction of graph posteriors we use the same labeled data $y$ and noise model $\phi^y:\R^p \to \R$ as in the continuum case ---see subsection \ref{ssec:dataandnoise}.

\subsubsection{Graph Posterior}\label{ssec:graphposterior}
We define the \emph{graph-posterior} measure $\mun$ by
\begin{align}\label{def:posteriordiscrete}
\begin{split}
\mun(du) \propto \exp\bigl(-\phi^y (\G_n u_n) \bigr) \pin(du_n) \\
=: \exp\bigl(-\Phi_n(u_n; y)\bigr) \pin(du_n),
\end{split}
\end{align}
where $\pin$ is the (truncated) graph prior defined as the law of \eqref{def:graphprior}, and the above expression defines the function $\Phi_n$, interpreted as a graph-based approximation to the negative log-likelihood. 

In subsection \ref{ssec:continuumlimits} we will contrast the above ``truncated" graph-posteriors to the ``untruncated" graph-posteriors 
\begin{align}\label{def:posteriordiscretetruncated}
\begin{split}
\mut(du) \propto \exp\bigl(-\phi^y (\G_n u_n) \bigr) \pit(du_n) \\
=: \exp\bigl(-\Phi_n(u_n; y)\bigr) \pit(du),
\end{split}
\end{align}
obtained by using the prior $\pit$ in equation \eqref{priordiscrete}.

\section{Posterior Sampling: pCN and Graph-pCN}\label{sec:sampling}
The continuum limit theory developed in \cite{garcia-sanz2017continuum} and recalled in subsection 
\ref{ssec:continuumlimits} suggests viewing graph posteriors $\mun$ as discretizations of a posterior measure over functions on the underlying manifold. Again, these discretizations are robust for fixed $p$ and growing number of total inputs $n$. This observation substantiates the idea introduced in \cite{bertozziuncertainty} of using a version of the pCN MCMC method \citep{beskos2008mcmc} for robust sampling of graph posteriors. We review the continuum pCN method  in subsection \ref{ssec:samplingcontinuum}, and the graph pCN counterpart in subsection \ref{ssec:samplinggraph}.

\subsection{Continuum pCN}\label{ssec:samplingcontinuum}

In practice, sampling of functions on the continuum always requires a discretization of the infinite dimensional function, usually defined  in terms of a mesh and possibly a series truncation. A fundamental idea is that  algorithmic robustness with respect to discretization refinement can be guaranteed by ensuring that the algorithm is well defined in function space, before discretization \citep{DS15}. This insight led to the formulation of the pCN method for sampling of conditioned diffusions \citep{beskos2008mcmc}, and of measures arising in Bayesian nonparametrics  in \cite{cotter2009bayesian}. The pCN method for sampling the continuum posterior measure \eqref{def:posterior} is given in Algorithm \ref{algorithm1}.

\begin{algorithm}[H]
\begin{algorithmic}
\BState Set $j = 0$ and pick any $u^{(0)} \in L^2(\gamma).$
\BState Propose  $\tilde{u}^{(j)} = (1-\beta^2)^{1/2} u^{(j)} + \beta \zeta^{(j)}, \quad \text{where}\,\, \zeta^{(j)} \sim N(0, \C_u)$.
\BState Set $u^{(j+1)} = \tilde{u}^{(j)}$ with probability $$a\bigl(u^{(j)}, \tilde{u}^{(j)}\bigr):=  \min \Bigl\{1, \exp\Bigl(\Phi\bigl(u^{(j)};y\bigr) - \Phi\bigl(\tilde{u}^{(j)};y\bigr) \Bigr)\Bigr\}.$$
\BState Set $u^{(j+1)} = u^{(j)}$ otherwise.
\BState $j\to j+1.$
\end{algorithmic}
\caption{Continuum pCN}\label{algorithm1}
\end{algorithm}

Posterior expectations of suitable test functions $f$ can then be approximated by empirical averages

\begin{equation}\label{eq:empiricalaverage}
\muu(f) \approx \frac{1}{J} \, \sum_{j=1}^J f\bigl(u^{(j)} \bigr) = S^J(f).
\end{equation}

The user-chosen parameter $\beta\in [0,1]$ in Algorithm \ref{algorithm1} monitors the step-size of the chain jumps: larger $\beta$ leads to larger jumps, and hence to more state space exploration, more rejections, and slower probing of high probability regions. Several robust discretization properties of Algorithm \ref{algorithm1} ---that contrast with the deterioration of traditional random walk approaches--- have been proved in \cite{hairer2014spectral}. Note that the acceptance probability is determined by the potential $\Phi$ (here interpreted as the negative log-likelihood) that defines the density of the posterior with respect to the prior. In the extreme case where $\Phi$ is constant, moves are always accepted. However, if the continuum posterior is far from the continuum prior, the density will be far from constant. This situation may arise, for instance, in cases where $p$ is large or the size $\sigma$ of the observation noise is small. A way to make posterior informed proposals that may lead to improved performance in these scenarios has been proposed in \cite{rudolf2015generalization}.

\subsection{Graph pCN}\label{ssec:samplinggraph}
The graph pCN method is described in Algorithm \ref{pCN-RRW. }, and is defined in complete analogy to the continuum pCN, Algorithm \ref{algorithm1}. When considering a sequence of problems with fixed $p$ and increasing $n,$ the continuum theory intuitively supports the robustness of the method. Moreover, as indicated in \cite{bertozziuncertainty} the parameter $\beta$ may be chosen independently of the value of $n.$ Our experiments in section \ref{sec:Numerics} confirm this robustness, and also investigate the deterioration of the acceptance rate when both $n$ and $p$ are large. 

\begin{algorithm}[H]
\caption{Graph pCN}\label{pCN-RRW. }
\begin{algorithmic}
\BState Set $j = 0$ and pick any $u_n^{(0)} \in L^2(\gamma_n).$
\BState Propose  $\tilde{u}_n^{(j)} = (1-\beta^2)^{1/2} u_n^{(j)} + \beta \zeta_n^{(j)}, \quad \text{where}\,\, \zeta_n^{(j)} \sim N(0, \C_{u_n})$.
\BState Set $u_n^{(j+1)} = \tilde{u}_n^{(j)}$ with probability
$$a_n\bigl(u_n^{(j)}, \tilde{u}_n^{(j)}\bigr):=  \min \Bigl\{1, \exp\Bigl(\Phi_n\bigl(u_n^{(j)};y\bigr) - \Phi_n\bigl(\tilde{u}_n^{(j)};y\bigr) \Bigr)\Bigr\}.$$
\BState Set $u_n^{(j+1)} = u_n^{(j)}$ otherwise.
\BState $j\to j+1.$
\end{algorithmic}
\end{algorithm}

 Again, graph-posterior expectations of suitable test functions $f_n$ can then be approximated by empirical averages
\begin{equation}\label{eq:empiricalaveragegraph}
\mut(f_n) \approx \frac{1}{J} \, \sum_{j=1}^J f_n\bigl(u^{(j)}_n \bigr) = S^J(f_n).
\end{equation}
In \emph{informal} but intuitive terms, the uniform spectral gap that we establish below shows that the large $J$ asymptotic variance of $S^J(f_n)$ is independent of $n$. 
\nc
\section{Main Results}\label{sec:theory}

\subsection{Continuum Limits}\label{ssec:continuumlimits}
The paper \cite{garcia-sanz2017continuum}  established large $n$ asymptotic convergence of the untruncated graph-posteriors $\mut$ in \eqref{def:posteriordiscretetruncated} to the continuum posterior $\muu$ in \eqref{def:posterior}. The convergence was established in a topology that combines Wasserstein distance and an $L^2$-type term in order to compare measures over functions in the continuum with measures over functions in graphs. 
\begin{proposition}[Theorem 4.4 in \cite{garcia-sanz2017continuum}] \label{proposition} Suppose that $s>2m$ and that 
 \begin{equation}\label{eq:epsilonchoice1}
 \frac{(\log(n))^{p_m} }{n^{1/m}}   \ll \veps_n \ll \frac{1}{n^{1/s}}  , \quad \text{ as } n \rightarrow \infty,
 \end{equation}
 where $p_m = 3/4$ for $m=2$ and $p_m = 1/m$ for $m\ge 3.$
 Then, the untruncated graph-posteriors $\mut$ converge towards the posterior $\muu$ in the $\P(TL^2)$ sense. 
\end{proposition}

We refer to Appendix \ref{TL2PTL2} for the construction of the metric space $TL^2$ that was originally introduced in \cite{trillos}. Notice that in the space $TL^2$ we can compare functions defined on $\M_n$ with functions defined on $\M$. The space $\mathcal{P}(TL^2)$ was introduced in 
\cite{garcia-sanz2017continuum} and stands for the set of Borel probability measures on $TL^2$ endowed with the topology of weak convergence. This space allows us to formalize the convergence of a sequence of probability distributions over functions on $\M_n$ to a probability distribution over functions on $\M$. In particular, in the previous theorem, convergence is interpreted as: $\mut$ converges weakly to $\muu$ as $n \rightarrow \infty,$ all measures seen as elements of $\mathcal{P}(TL^2)$. 
It is important to note that in the theorem, convergence refers to the limit of \emph{fixed} labeled data set of size $p$, and growing size of unlabeled data. In order for the continuum limit to hold, the connectivity of the graph $\veps_n$ needs to be carefully scaled with $n$ as in \eqref{eq:epsilonchoice1}.

At an intuitive level, the lower bound on $\veps_n$ guarantees that there is enough averaging in the limit to recover a meaningful deterministic quantity. The upper bound ensures that the graph priors converge appropriately towards the continuum prior. At a deeper level, the lower bound is an order one asymptotic estimate for the $\infty$-optimal transport distance between the uniform and uniform empirical measure on the manifold \citep{trillos2014canadian}, that hinges on the points $\x_1, \dots, \x_n$ lying \emph{on} the manifold $\M$: if the inputs were sampled from a distribution whose support is close to $\M$, but whose intrinsic dimension is $d$ and not $m$, then the lower bound would be written in terms of $d$ instead of $m$. The upper bound, on the other hand, relies on the approximation bounds \eqref{eq:estimateseigen} of the continuum spectrum of the Laplace-Beltrami by the graph Laplacian. 

We now present a new result on the stability of intrinsically constructed posteriors, generalized to $\M$ by interpolation via the map $\mathcal{I}_n :=\mathcal{I}_n^1$ ---see \eqref{def:I}; this is the most basic interpolation map that can be constructed exclusively from the point cloud $\M_n$ and the metric on the ambient space. Other than extending the theory to cover the important question of generalization, there is another layer of novelty in Theorem \ref{th:interpolant}: graph-posteriors are constructed with truncated priors, and the upper-bound in the connectivity $\varepsilon_n$ in \eqref{eq:epsilonchoice1} is removed.  As discussed in subsection \ref{ssec:approxbounds}, only a portion of the spectrum of the graph Laplacian contains relevant information of the underlying manifold $\M$, and thus nothing is lost by throwing away higher modes. See Figure \ref{fig:Laplacian Spectrum} for an illustration.  \nc

\begin{theorem}\label{th:interpolant}
Suppose that $s>2m$ and that 
 \begin{equation}\label{eq:epsilonchoice}
 \frac{(\log(n))^{p_m} }{n^{1/m}}   \ll \veps_n \ll 1 , \quad \text{ as } n \rightarrow \infty,
 \end{equation}
 where $p_m$ is as in Proposition \ref{proposition}.
	Then, with probability one, 
	\[ \mathcal{I}_{n \sharp } \mun \rightarrow_{\mathcal{P}(L^2(\gamma))} \muu, \quad \text{ as } n \rightarrow \infty.\]
\end{theorem}
	The proof is presented in Appendix \ref{AppA}. Similar results hold for more general interpolation maps as long as they are uniformly controlled and consistent when evaluated at the eigenfunctions of graph Laplacians (see Remark \ref{GeneralInterpolants}).

\begin{remark}
\label{remarkp1}
Our results concern the regime where $n \rightarrow \infty$ and  $p$ is constant. This corresponds to the semi-supervised setting of many more unlabeled data points than labels. Our analysis would also allow us to take the double limit $n \rightarrow \infty $ followed by $p \rightarrow \infty$. This corresponds to a semi-supervised learning regime where both the number of unlabeled data points and the number of labeled data points grow, but $p$ grows at the slowest rate possible. In that regime the limiting posterior concentrates around a single function on $\M$ which would correspond to the true ``regression function". It may be possible to establish similar posterior concentration results in the regime where both   $n \rightarrow \infty$ and $p=p_n \rightarrow \infty$ go simultaneously to infinity as well as to establish posterior contraction rates. We leave such analysis for future work.  
\end{remark}

	\subsection{Uniform Spectral Gaps for Graph-pCN Algorithms}\label{ssec:pcntheory}
	The aim of this subsection is to establish how, in a precise and rigorous sense, the graph-pCN method in Algorithm \ref{pCN-RRW. } is insensitive to the increase of the number $n$ of input data provided that the number $p$ of labeled data is fixed and that a continuum limit exists. This behavior contrasts dramatically with other sampling methodologies such as the random walk sampler. One could characterize the robustness of MCMC  algorithms in terms of uniform spectral gaps.
	
	We start by defining the spectral gap for a single Markov chain with state space an arbitrary separable Hilbert space $\mathcal{H}$. We consider two notions of spectral gap, one using Wasserstein distance with respect to some distance like function $\tilde{d},$ and the other one in terms of $L^2$.  For the purposes of this paper the Wasserstein spectral gap can be thought as an intermediate step which is ``easier" to prove directly following the ideas introduced in \cite{hairer2014spectral}, while the $L^2$ gap is a consequence whose implications are  meaningful for our problem. We start with the two definitions.

	\begin{definition}[Wasserstein spectral gaps]
		Let $P$ be the transition kernel for a discrete time Markov chain with state space $\mathcal{H}$. Let $\tilde{d}: \H \times \H \rightarrow [0,1]$ be a distance like function (i.e. a symmetric, lower-semicontinuous function satisfying $\tilde{d}(u,v)=0$ if and only if $u=v$). Without the loss of generality we also denote by $\tilde{d}$ the Wasserstein  distance (1-OT distance) on $\mathcal{P}(\H)$  induced by $\tilde{d}$ (see \eqref{Wass}). We say that $P$ has spectral gap if there exist positive constants $C, \lambda$ such that
		\[ \tilde{d}(P^j\mu, P^j \nu) \leq C \exp(-\lambda j) \tilde{d}(\mu, \nu), \quad \forall \mu, \nu \in \mathcal{P}(\H), \quad \forall j \in \N. \]
		In the above $\mathcal{P}(\H)$	stands for the set of Borel probability measures on $\H$.
	\end{definition}
	
\begin{definition}[$L^2$-spectral gaps]
Let $P$ be the transition kernel for a discrete time Markov chain with state space $\mathcal{H}$ and suppose that $\mu$ is invariant under $P$. $P$ is said to have $L^2_\mu$-spectral gap $1- \exp(-\lambda) $ (for $\lambda >0$) if for every $f \in L^2(\H; \mu)$ we have
\[  \frac{\lVert  P f  - \mu(f)   \rVert^2_{L^2(\H; \mu)} }  { \lVert  f  - \mu(f)   \rVert^2_{L^2(\H; \mu)}   } \leq \exp(-\lambda). \]
In the above, $\mu(f) := \int_{\H} f(u) d \mu(u) $ and $Pf(u):= \int_{\H}  f(v) P(u, dv)$.
\end{definition}

Having defined the notion of spectral gap for a single Markov chain, the notion of uniform spectral gap for a family of Markov chains is defined in an obvious way. Namely, if $\{ P_n \}_{n \in \N}$ is a family of Markov chains, with perhaps different state spaces $ \{ \mathcal{H}_n \}_{n \in \N}$, we say that the family of Markov chains has \textit{uniform} Wasserstein spectral gap with respect to a family of distance like functions $\{ \tilde{d}_n \}$ if the Markov chains have spectral gaps with constants $C, \lambda$ which can be uniformly bounded, respectively,  from above and away from zero. Likewise the chains are said to have uniform $L^2$-gaps (with respect to respective invariant measures) if the constant $\lambda$ can be uniformly bounded away from zero. We remark that Wasserstein spectral gaps imply uniqueness of invariant measures of Markov chains (this follows directly from the definition of Wasserstein gap).

Having introduced the above notions of ``mixing" for Markov chains in a general setting, we return to the problem of understanding the mixing of the family of pCN algorithms for our semi-supervised learning problem. We will make the following assumption on the negative log-likelihood function $\phi^y$.

	\begin{assumption}
		\label{Assumptionphi}
		Let $\beta\in (0,1]$. For a certain fixed $y \in \R^p$  we assume the following conditions on $\phi^y: \R^p \rightarrow \R$.
		\begin{enumerate}
			\item For every $K>0$ there exists $c\in \R$ such that if $v, w\in \R^p$ satisfy 
			\[  \lvert w - \sqrt{1-\beta^2}\; v \rvert \leq K  \]
			then,
			\[ \phi^y( v) - \phi^{y}(w) \geq c .\]
			\item (Linear growth of local Lipschitz constant) There exists a constant $L$ such that 
			\[   \lvert \phi^y( v ) - \phi^y (w)  \rvert \leq L\max \{ \lvert v \rvert, \lvert w \rvert ,1 \} \lvert  v -  w  \rvert, \quad \forall v, w \in \R^p.   \]
		\end{enumerate}
		
	\end{assumption}
	In Appendix \ref{App3} we show that the Gaussian model and the probit model satisfy these assumptions.

	In what follows it is convenient to use $\mathcal{H}$ as a placeholder for one of the spaces $L^2(\gamma_n),$ $n\in \N,$ or the space $L^2(\gamma)$. Likewise $P$ is a placeholder for the transition kernel associated to the pCN scheme from section \ref{sec:sampling} defined on $\mathcal{H}$ for each choice of $\mathcal{H}$. We are ready to state our second main theorem:

	\begin{theorem}[Uniform Wasserstein spectral gap]
		\label{th:pcn} Let $\theta>0$, $\eta>0$. For each choice of $\mathcal{H}$ let $d: \H \times \H \rightarrow [0, 1]$,
		\[ d(u, v) :=\min \Bigl\{1, \frac{\overline{d}(u,v)}{\theta}  \Bigr \}, \quad u,v \in \H   \]
		be a rescaled and truncated version of the distance 
		\[\overline{d}(u,v):= \inf_{T, \psi \in A(T,u,v)} \int_{0}^T \exp(\eta \lVert  \psi \rVert )dt ,\]
		\[ A(T,u,v):= \{ \psi \in C^1([0,T]; \H) \: : \: \psi(0)=u , \quad \psi(T)=v, \quad \lVert \dot{\psi} \rVert=1   \}. \]
		Finally, let $\tilde{d}$ be the distance-like function 
		\[  \tilde{d}(x,y):= \sqrt{d(x,y) ( 1+ V(x) + V(y)) }, \quad u,v \in \H \]
		where
		\[ V(u):= \lVert u \rVert^2, \quad u \in \H. \]
		Then, under the assumptions of Theorem \ref{th:interpolant} and Assumption \ref{Assumptionphi}, $\theta>0$ and $\eta>0$ can be chosen independently of the specific choice of $\mathcal{H}$ in such a way that
		\[  \tilde{d}(P^j \nu_1, P^j\nu_2) \leq C \exp(- \lambda j) \tilde{d}(\nu_1, \nu_2), \quad \forall \nu_1, \nu_2 \in \mathcal{P}(\H), \quad \forall j \in \N, \]
		for constants $C, \lambda$ that are independent of the choice of $\mathcal{H}$. 
	\end{theorem}
	
	A few remarks help clarify our results.

	\begin{remark}
	Notice that $\overline{d}$ is a Riemannian distance whose metric tensor changes in space and takes larger values for points that are far away from the origin (notice that the choice $\eta=0$ returns the canonical distance on $\H$). In particular,  points that are far away from the origin have to be very close in the canonical distance in order to be close in the $d$ distance. This distance was considered in \cite{hairer2014spectral}. We would also like to point out that the exponential form of the metric tensor can be changed to one with polynomial growth given the choice of $V$.
	\end{remark}

	\begin{remark}
	Theorem \ref{th:pcn} is closely related to Theorem 2.14 in \cite{hairer2014spectral}. There, uniform spectral gaps are obtained for the family of pCN kernels indexed by the truncation levels of the Karhunen Lo\`{e}ve expansion of the continuum prior. For that type of discretization, all distributions are part of the same space; this contrasts with our set-up where the discretizations of the continuum prior are the graph priors.  
	\end{remark}
	\nc
 
Due to the reversibility of the kernels associated to the pCN algorithms (they are particular instances of Metropolis-Hastings), Theorem \ref{th:pcn} implies uniform  $L^2$-spectral gaps as introduced earlier. Notice that the Wasserstein gaps imply uniqueness of invariant measures (which are precisely the graph and continuum posteriors for each setting) and hence there is no ambiguity when talking about $L^2$-spectral gaps.
	
	\begin{corollary}
	\label{corGAPS}
	Under the assumptions of Theorem \ref{th:interpolant} and Assumption \ref{Assumptionphi} the kernel associated   to the pCN algorithm has an $L^2$-spectral gap independent of the choice of $\H$.
	\end{corollary}
	The proof of Theorem \ref{th:pcn} and its corollary are presented in Appendix \ref{AppB}.

Recall that graph-posterior expectations of suitable test functions $f_n$ can be approximated by empirical averages
\begin{equation}
\mut(f_n) \approx \frac{1}{J} \, \sum_{j=1}^J f_n\bigl(u^{(j)}_n \bigr) = S^J(f_n).
\end{equation}
Roughly speaking, this uniform spectral gap shows that the large $J$ asymptotic variance of $S^J(f_n)$ is independent of $n$.  Uniform spectral gaps may be used to find uniform bounds on the asymptotic variance of empirical averages \citep{kipnis1986central}.

	\begin{remark}
	It is important to highlight that the uniform gaps for the pCN algorithm (when $n$ grows) depend nonetheless on the number of observations $p$, and that the gaps may collapse with growing $p$. This should be intuitively reasonable as this corresponds to considering a more complex likelihood function, which in turn pushes the posterior further from the prior.   	
	\end{remark}

	\section{Numerical Study}\label{sec:Numerics}
	In the numerical experiments that follow we take $\M= \S$ to be the two-dimensional sphere in $\R^3.$ Our main motivation for this choice of manifold is that it allows us to expediently make use of well-known closed formulae \citep{olver2013introduction} for the spectrum of the spherical Laplacian $\Delta_\M = \Delta_\S$ in the continuum setting that serves as our ground truth model. We recall that $-\Delta_\S$ admits eigenvalues  $l(l+1)$, $l \ge 0,$ with corresponding eigenspaces of dimension $2l+1$. These eigenspaces are spanned by spherical harmonics \citep{olver2013introduction}.  In subsections \ref{ssec:spectrumnumerics}, \ref{ssec:numericsposterior}, and \ref{ssec:pcnrobustness} we study, respectively, the spectrum of graph Laplacians, continuum limits, and the scalability of pCN methods.

	\subsection{Spectrum of Graph Laplacians}\label{ssec:spectrumnumerics}
The asymptotic behavior of the spectra of graph-Laplacians is crucial in the theoretical study of consistency of graph-based methods. In subsection \ref{ssec:approxbounds} we review approximation bounds that motivate our truncation of graph-priors, and in subsection \ref{ssec:regularity} we comment on the theory of regularity of functions on graphs. 
	\subsubsection{Approximation Bounds}\label{ssec:approxbounds}
	Quantitative error bounds for the difference of the spectrum of the graph Laplacian and the spectrum of the Laplace-Beltrami operator are given in \cite{burago2014graph} and \cite{SpecRatesTrillos}. Those results imply that, with very high probability, 	
	\begin{equation}
	\Bigl\lvert 1- \frac{\lambda_i^n}{\lambda_i} \Bigr\rvert \leq C \Bigl(\frac{\delta_n}{\veps_n} + \veps_n\sqrt{\lambda_i}\Bigr), \quad \forall i,
	\label{eq:estimateseigen}
	\end{equation}
	where $\delta_n$ denotes the $\infty$-optimal transport distance \citep{trillos2014canadian} between the uniform and the uniform empirical measure on the underlying manifold. The important observation here is that the above estimates are only relevant for the first portion of the spectra (in particular for those indices $i$ for which $\veps_n \sqrt{\lambda_i}$ is small). The truncation point at which the estimates stop being meaningful can then be estimated combining \eqref{eq:estimateseigen} and Weyl's formula for the growth of eigenvalues of the Laplace Beltrami operator on a compact Riemannian manifold of dimension $m$ \citep{garcia-sanz2017continuum}. Namely, from $ \lambda_i \sim i^{2/m} $ we see that $\veps_n \sqrt{\lambda_i} \ll 1$ as long as $i=1, \dots, k_n$ and
\[   1 \ll k_n \ll \frac{1}{\veps_n^m } . \]
 This motivates our truncation point for graph priors in equation \eqref{def:graphprior}.

Figure \ref{fig:Laplacian Spectrum}  illustrates the approximation bounds \eqref{eq:estimateseigen}. The figure shows the eigenvalues of the graph Laplacian for three different choices of connectivity length scale $\veps$ and three different choices of number $n$ of inputs in the graph; superimposed is the spectra of the spherical Laplacian. We notice the flattening of the spectra of the graph Laplacian and, in particular, how the eigenvalues of the graph Laplacian start deviating substantially from those of the Laplace-Beltrami operator after some point in the $x$-axis. As discussed in \cite{SpecRatesTrillos}, the estimates \eqref{eq:estimateseigen} are not necessarily sharp, and may be conservative in suggesting where the deviations start.

	
	\begin{figure}
		\includegraphics[width=\linewidth]{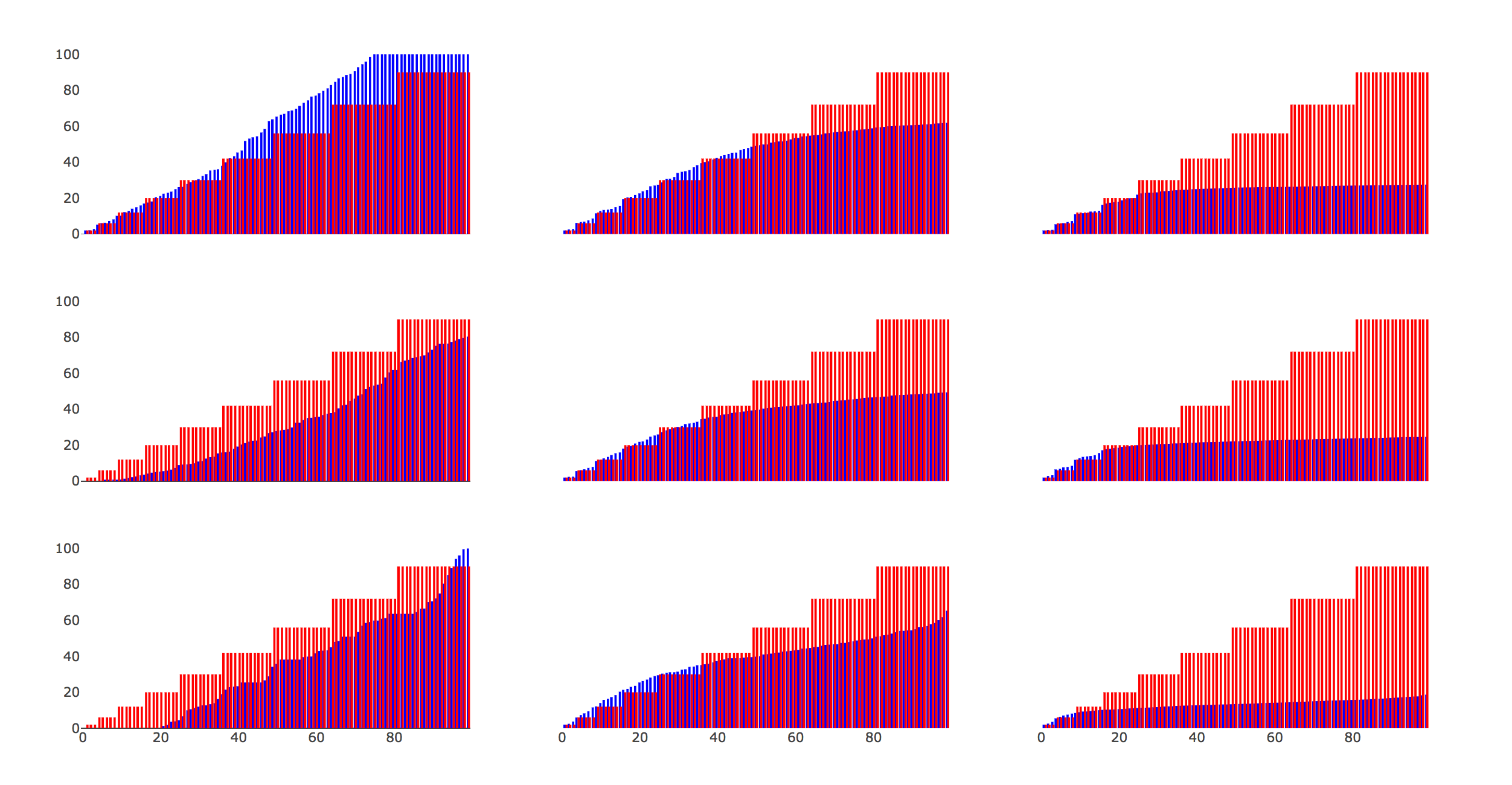}
		\caption{Spectra of spherical and graph Laplacians in red and blue, respectively. Charts are arranged such that $\varepsilon$ varies as $ [1,\  2,\  3] \times n^{-1/4} $ horizontally and $n$ varies as $[1000,\ 500,\ 100]'$ vertically.}
		\label{fig:Laplacian Spectrum}
	\end{figure}

	\subsubsection{Regularity of Discrete Functions}\label{ssec:regularity}
	We numerically investigate the role of the parameter $s$ in the discrete regularity of functions $u_n \in L^2(\gamma_n)$ sampled from $\pin$. We focus on studying the oscillations of a function within balls of radius $\veps_n$. More precisely, we consider 
	\[  [\Osc_{\veps_n}(u_n)](\x_i):= \max_{ x,z \in B_{\veps_n}(\x_i ) \cap \M_n }\lvert u_n(x) - u_n(z)  \rvert, \quad i=1, \dots, n.  \]

For given $s=2,3, \dots, 8$ we take $100$ samples $u_n \sim \pin,$ and we normalize so that
	\[ \langle \Delta_n^{s} u_n , u_n \rangle_{L^2(\gamma_n)}=1 .\] 
	We then compute the maximum value of $[\Osc_{\veps_n}(u_n)](\x_i)$ over all $i=1,\dots,n$ and over all samples $u_n$ and plot the outcome against $s$. The results are shown in Figure \ref{fig: Priors regularity}.
		\begin{figure}
		\includegraphics[width=\linewidth]{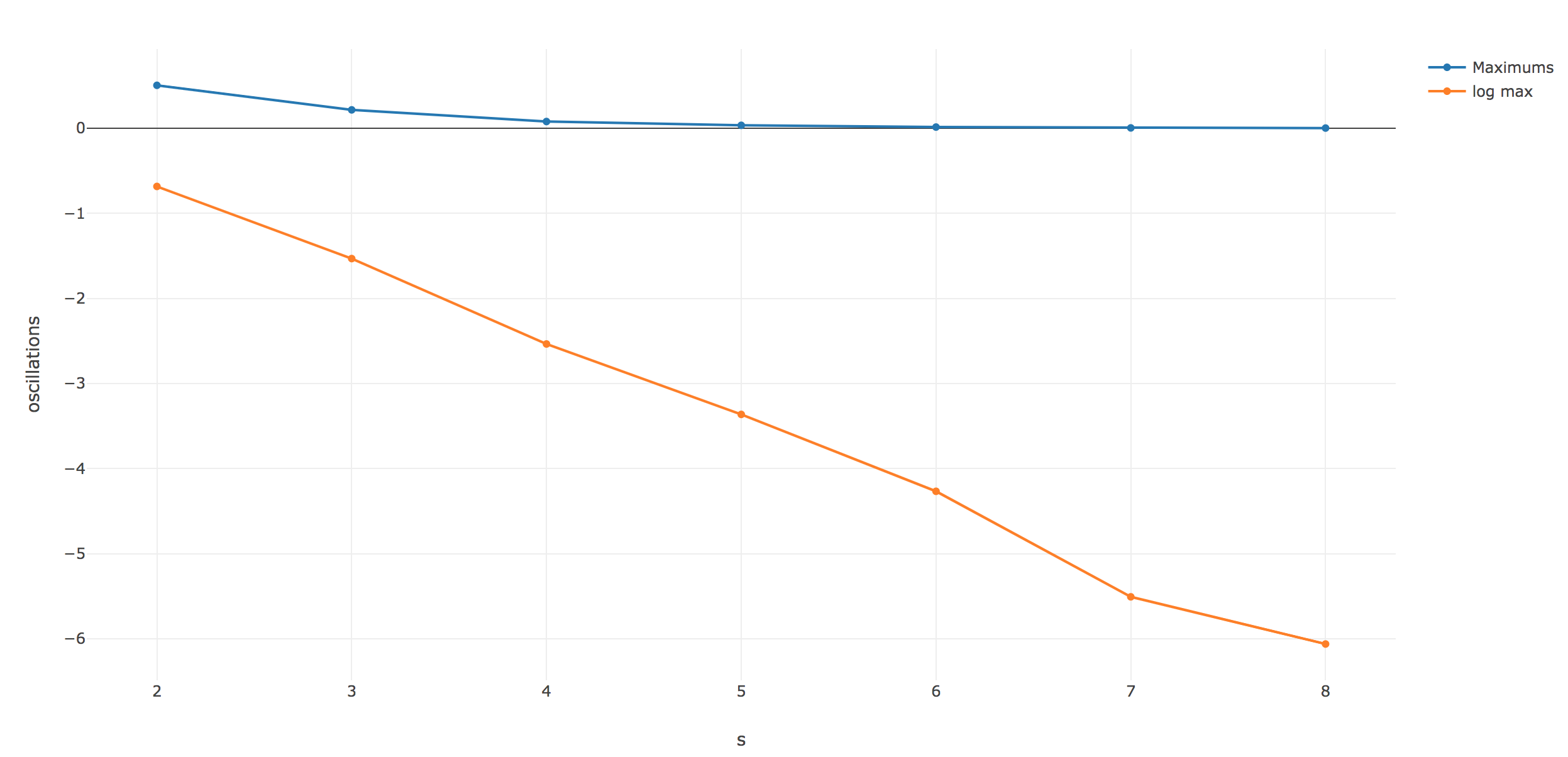}
		\caption{The figure shows the maximum (and its logarithm) amplitude of oscillations for different values of the regularity parameter $s$.}
		\label{fig: Priors regularity}  	
	\end{figure}
This experiment illustrates the regularity of functions with bounded $H_n^{s}$ semi-norm 
\[ \lVert u_n \rVert_{H_n^{s}}^2:= \sum_{i=1}^{k_n} (\lambda_i^n)^s \langle u_n , \psi_i^n \rangle_{L^2(\gamma_n)}^2.\]
As expected, higher values of $s$ enforce more regularity on the functions. Notice that here we only consider functions $u_n$ in the support of $\pin$ and hence we remove the effect of high eigenfunctions of $\Delta_n$ (which may be irregular). In particular, the regularity of the functions $u_n$ must come from the regularity of the first eigenvectors of $\Delta_n$ together with the growth of $(\lambda_i^n)^s$. To the best of our knowledge nothing is known about regularity of eigenfunctions of graph Laplacians. Studying such regularity properties is an important direction to explore in the future as we believe it would allow us to go beyond the $L^2$ set-up that we consider for the theoretical results in this paper. In that respect we would like to emphasize that the observation maps considered for the theory of this work are defined in terms of averages and not in terms of pointwise evaluations, but that for our numerical experiments we have used the latter. 

A closely related setting in which discrete regularity has been mathematically studied is in the context of \textit{graph $\mathtt{p}$-Laplacian semi-norm} (here $\mathtt{p}$ denotes an arbitrary number greater than one, and is not to be confused with the number $p$ of labeled data points).  Lemma 4.1 in \cite{slepvcev2017analysis} states that, under the assumptions on $\veps_n$ from Theorem \ref{th:interpolant}, for all large enough $n$ and for every discrete function $u_n$ satisfying
	\[  \mathcal{E}_n^{(\mathtt{p})}(u_n):= \frac{1}{n^2 \veps_n^\mathtt{p}} \sum_{i,j} K \left(\frac{\lvert \x_i-\x_j \rvert}{\veps_n}\right) \lvert u_n(\x_i) - u_n(\x_j) \rvert^{\mathtt{p}}   =1, \]
	it holds
\[ [\Osc_{\veps_n}(u_n)](\x_i) \leq C^{1/\mathtt{p}} \,n^{1/\mathtt{p}} \, \veps_n, \quad \forall i=1, \dots, n. \]
This estimate allows to establish uniform convergence (and not simply convergence in $TL^2$) of discrete functions towards functions defined at the continuum level. More precisely, suppose that $\mathtt{p}>m$ and that $\veps_n \ll \frac{1}{n^{1/\mathtt{p}}}$. Let $\{ u_n \}_{n \in \N}$ be a sequence with $u_n \in L^2(\gamma_n)$ converging to a function $u \in L^2(\gamma)$ in the $TL^2$ sense and for which
\[ \sup_{n \in \N} \mathcal{E}_n^{(\mathtt{p})}(u_n) < \infty. \]
Then, $u$ must be continuous (in fact H\"{o}lder continuous with H\"{o}lder constant obtained from the Sobolev embedding theorem) and moreover
\[ \max_{i=1, \dots, n} \lvert  u_n(\x_i) - u(\x_i) \rvert  \rightarrow 0, \quad \text{as } n \rightarrow \infty.\]
This is the content of Lemma 4.5 in \cite{slepvcev2017analysis}. This type of result rigorously justifies pointwise evaluation of discrete functions with bounded graph $\mathtt{p}$-Laplacian seminorm and the stability of this operation as $n \rightarrow \infty$.

	\subsection{Continuum Limits}\label{ssec:numericsposterior}
	
	\subsubsection{Set-up}\label{ssec:numericsset-up}
	For the remainder of section \ref{sec:Numerics} we work under the assumption of Gaussian observation noise, so that 
	\begin{equation}\label{eq:gaussiannoise}
	\Phi(u;y) = \frac{1}{2\sigma^2}|y-\G(u)|^2, \quad \Phi_n(u_n,y) = \frac{1}{2\sigma^2} |y - \G_n(u_n)|^2.
	\end{equation}
	The synthetic data $y$ in our numerical experiments is generated by drawing a sample $\eta\sim N(0, \sigma^2I_{p\times p})$, and setting 
	$$y = \G(u^\dagger) + \eta,$$
	where $u^\dagger$ is the function in the left panel of Figure \ref{fig:Draws From Continuous Prior}. We consider several choices of $t\ge 0,$ number $p$ of labeled data points, and size of observation noise $\sigma>0.$ The parameters $s$ and $\alpha$ in the prior measures are fixed to $s=5$, $\alpha=1$ throughout. 
	
	The use of Gaussian observation noise, combined with the linearity of our forward and observation maps, allows us to derive closed formulae for the graph and continuum posteriors. We do so in the the appendix. 
	
	\begin{figure}
		\includegraphics[scale=0.35]{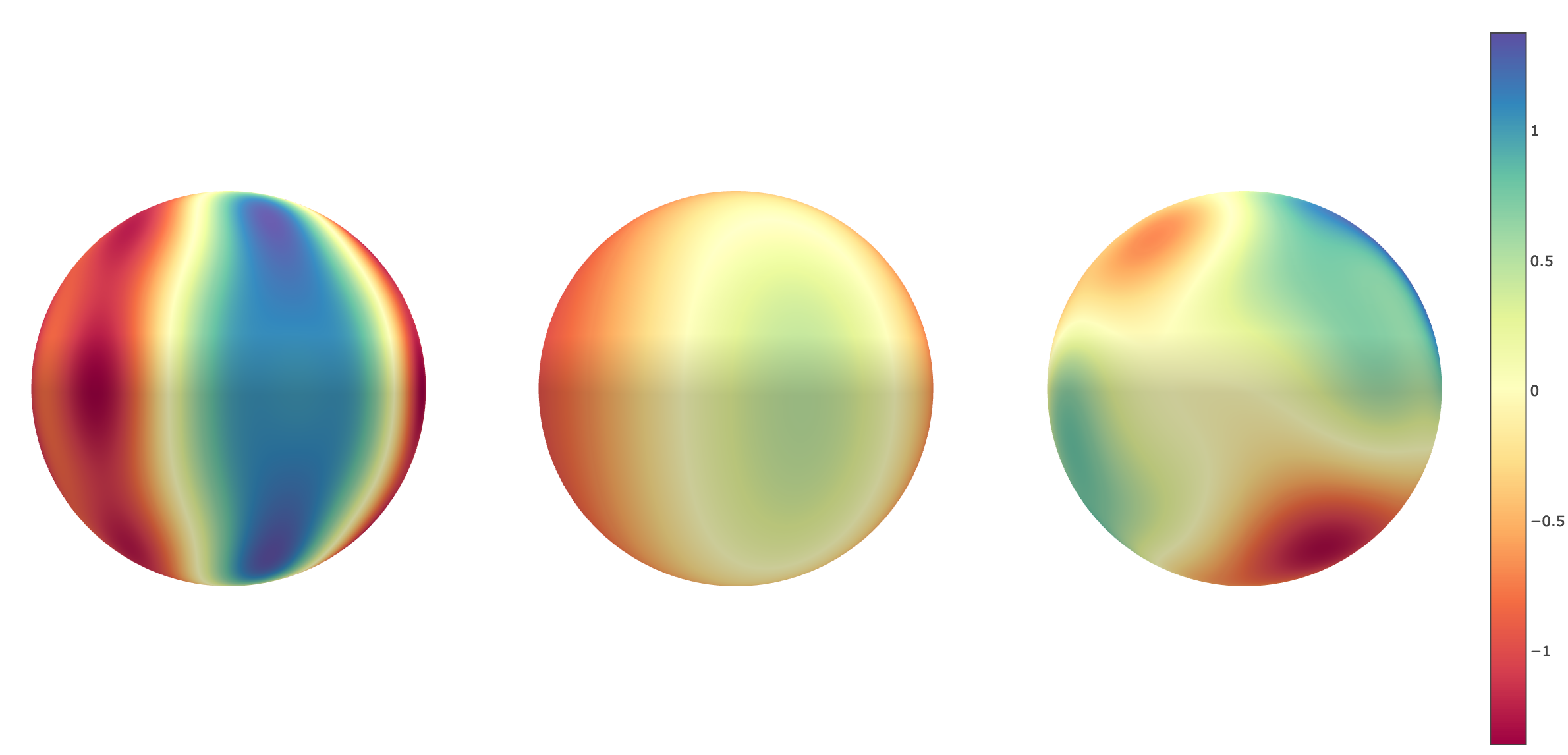}
		\caption{From left to right: Initial condition $u^\dagger$ used as ground truth to generate our synthetic data; heat at $t=0.5$ with initial condition $u^\dagger$; and draw from the continuous prior.}
		\label{fig:Draws From Continuous Prior}
	\end{figure}

	\subsubsection{Numerical Results}
	Here we complement the theory by studying the effect that various model parameters have in the accurate approximation of continuum posteriors by graph posteriors. We emphasize that the continuum posteriors serve as a gold standard for our learning problem: graph posteriors built with appropriate choices of connectivity $\varepsilon$ result in good approximations to continuum posteriors; however, reconstruction of the unknown function $u^\dagger$ is \emph{not} accurate if the data is not informative enough. In such case, MAPs constructed with graph or continuum posteriors may be far from $u^\dagger.$ 
	
	All graph-posterior means in the figures are represented using a $k$-NN interpolation map, as defined in equation \eqref{def:I}, with $k=4.$  The posterior means, discrete and continuum, have been obtained using the appropriate pCN algorithm. The pCN algorithm was run for $10^5$ iterations, and the last $10^4$ samples were used to compute quantities of interest (e.g means and variances). Figure \ref{fig:Draws From Discrete Prior} shows a graph-prior draw represented in the point cloud (left), and the associated $4$-NN interpolant (right).

		\begin{figure}
		\begin{center}
		\includegraphics[scale=0.35]{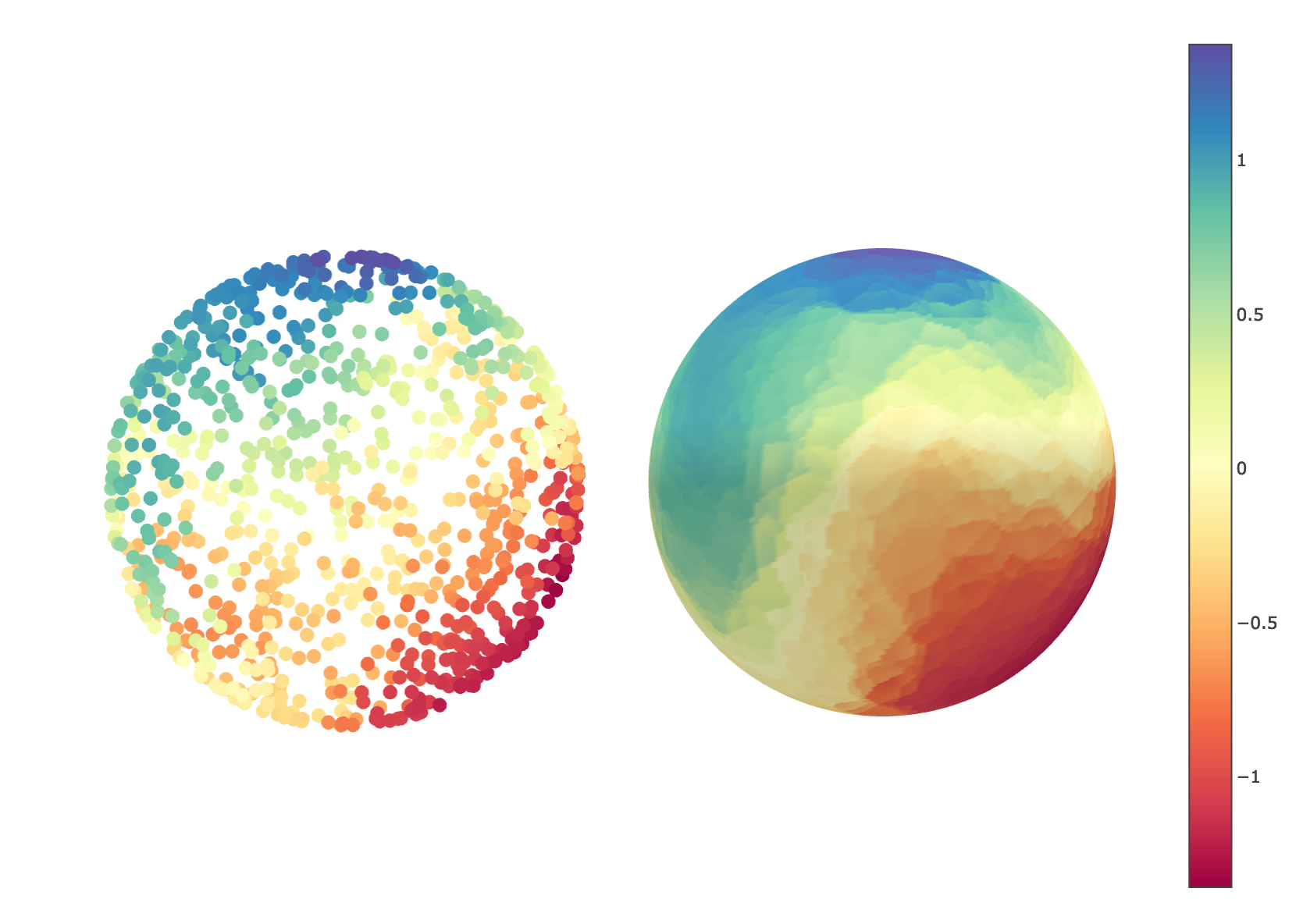}
		\caption{Draw from the discrete graph prior on the left, and the corresponding representation visualized using a 4-nearest-neighbors interpolation on the right. Parameters are $\varepsilon = 2n^{-1/4} $, $n=1000.$}
		\label{fig:Draws From Discrete Prior}
		\end{center}
	\end{figure}
	
	Figure \ref{fig:Discrete Posterior Differences} shows graph and continuum posteriors with $t=0,$ $t=0.1,$ and $t=0.3.$ For these plots, a suitable choice of graph connectivity $\varepsilon$ was taken. In all three cases we see remarkable similarity between the graph and continuum posterior means.  However, recovery of the initial condition with $t=0.3$ is unsuccessful: the data does not contain enough information to accurately reconstruct $u^\dagger$. Figure \ref{fig:Discrete Posterior Differences Epsilon} shows graph-posterior means computed in the regime of the first row of Figure \ref{fig:Discrete Posterior Differences} using the three graphs in Figure \ref{fig:graphs}. Note that the spectra of the associated graph-Laplacians is represented in Figure \ref{fig:Laplacian Spectrum}. It is clear that inappropriate choice of $\varepsilon$ leads to poor approximation of the continuum posterior, and here also to poor recovery of the initial condition $u^\dagger.$ This is unsurprising in view of the dramatic effect of the choice of $\varepsilon$ in the approximation properties of the spectrum of the spherical Laplacian, as shown in Figure \ref{fig:Laplacian Spectrum}. Note that while the numerical results are outside the asymptotic regime ($n=1000$ throughout), they illustrate the role of $\varepsilon.$ Theorem \ref{th:interpolant} establishes appropriate scalings for successful graph-learning in the large $n$ asymptotic setting.
	
	\begin{figure}
	\begin{center}
		\includegraphics[scale=.40]{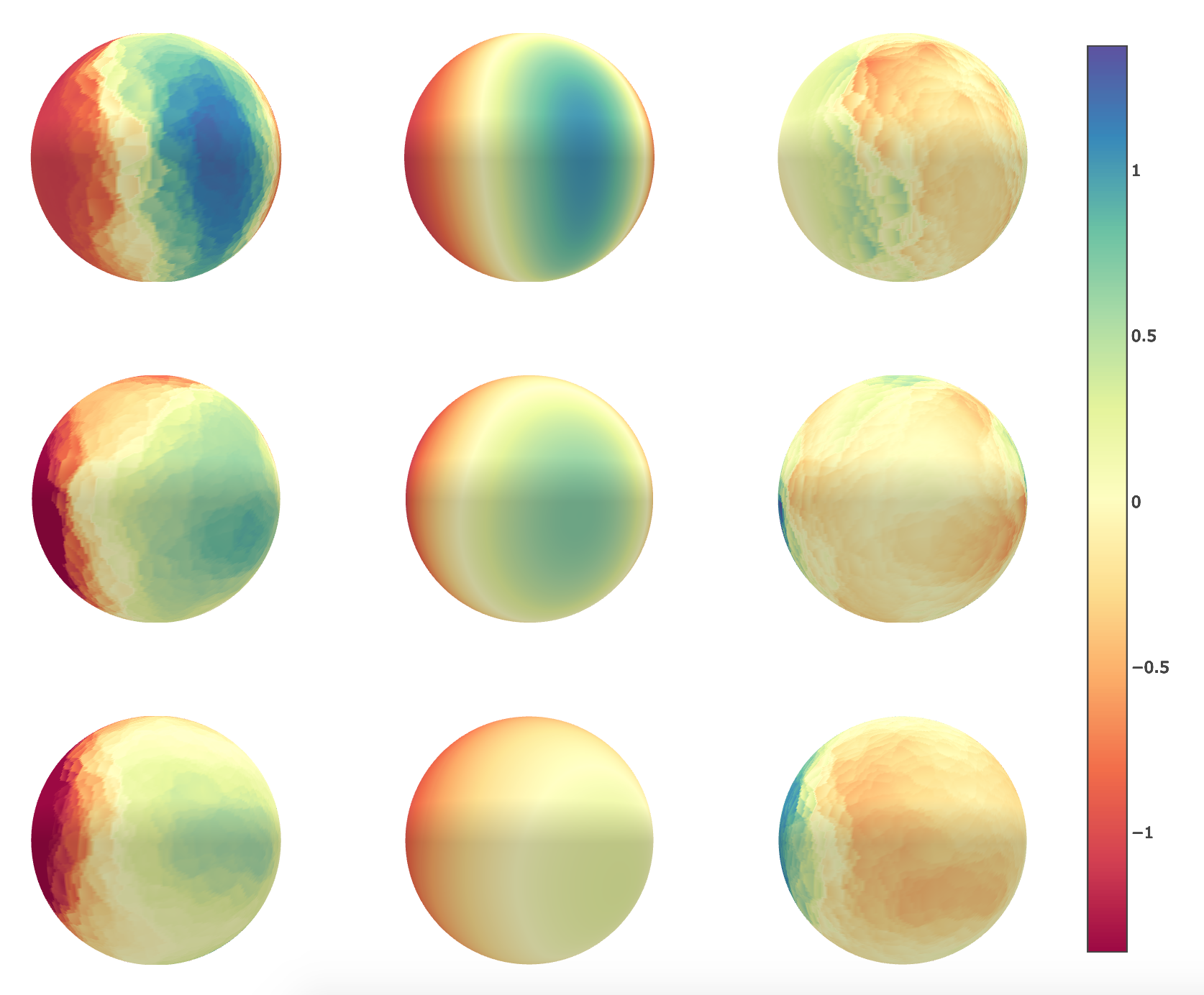}
		\caption{Means of the discrete and continuum posterior distributions are compared; the plots $P_{i,j}$ are arranged such that $P_{i,1}$ are graph-posterior means, $P_{i,2}$ are continuum posterior means, and $P_{i,3}$ are the differences row-wise. $P_{1,j},\ P_{2,j},\ P_{3,j}$ differ in the choice of the time parameter. They are, from the top, $t = [0,\ 0.1,\ 0.3]$.} 
		\label{fig:Discrete Posterior Differences}
		\end{center}
	\end{figure}
	
	\begin{figure}
	\begin{center}
		\includegraphics[scale=.32]{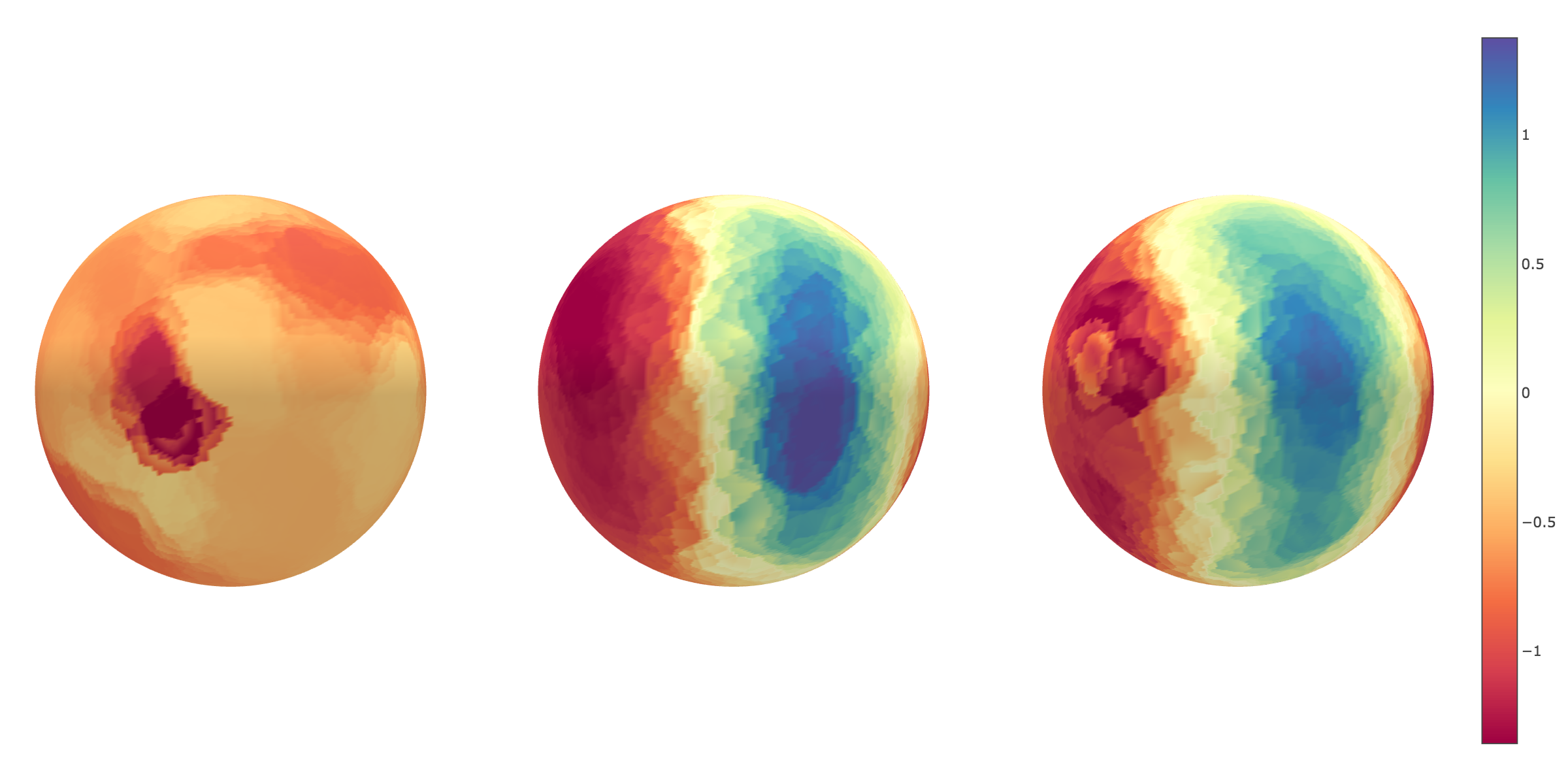}
		\caption{Graph-posterior means computed with the graph-pCN algorithm. All parameters of the learning problem are fixed to $t = 0,\ \sigma = 0.1,\ n=1000,\ \text{and } p=200$. The three plots show three choices of graph connectivities  $\varepsilon =  [1,\  2,\  3] \times n^{-1/4}$ as in Figure \ref{fig:graphs}.} \label{fig:Discrete Posterior Differences Epsilon}
	\end{center}
	\end{figure}
	
	\begin{figure}
		\includegraphics[scale=0.35]{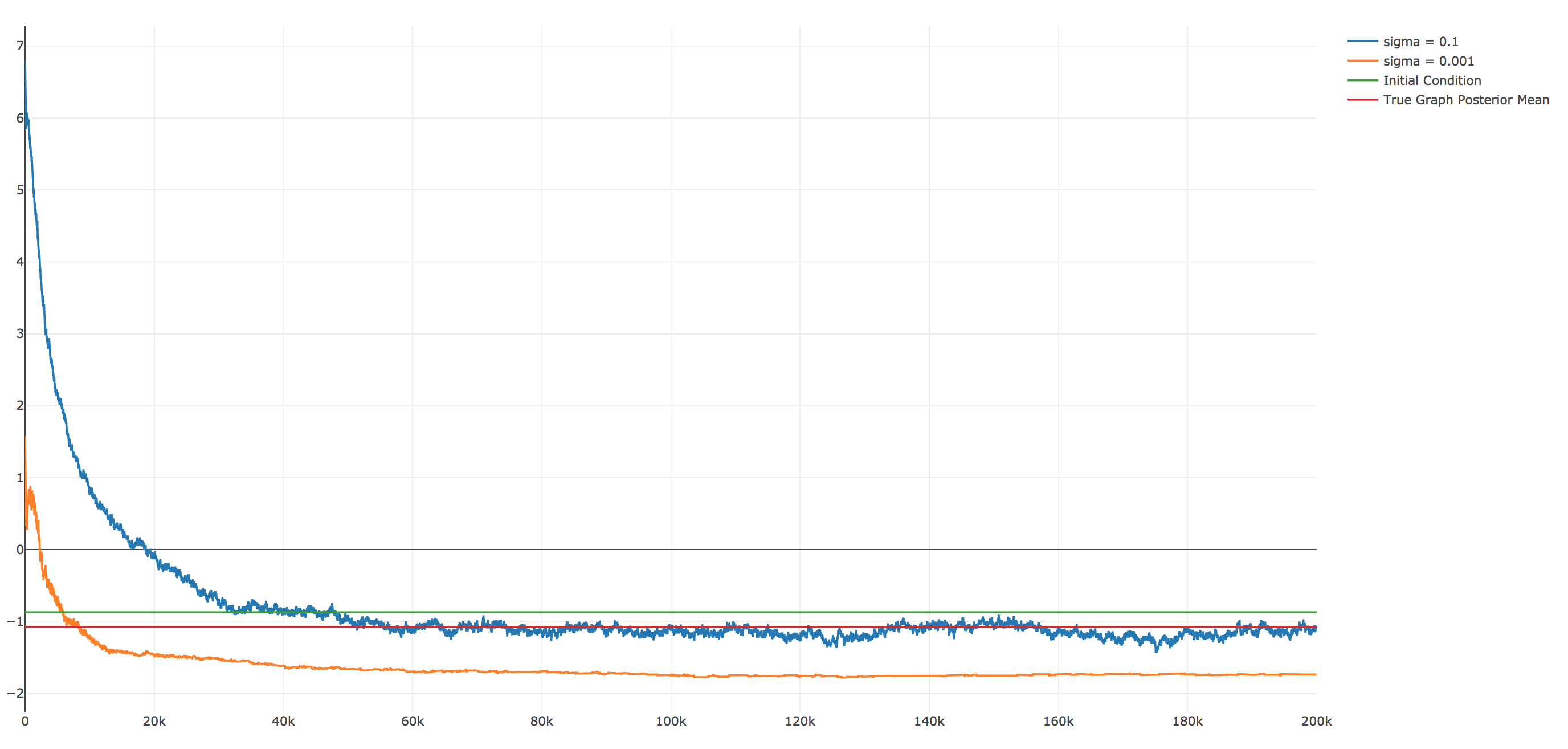}
		\caption{Effect of the parameter $\sigma$ on graph-pCN algorithm. When $\sigma$ is prohibitively small, here $\sigma=0.001$, the chain fails to mix rapidly. With more noise, here $\sigma = 0.1$, the chain mixes rapidly. } 
		\label{fig:Mixing chain different sigmas}
	\end{figure}
	
	\begin{figure}
		\includegraphics[scale=0.35]{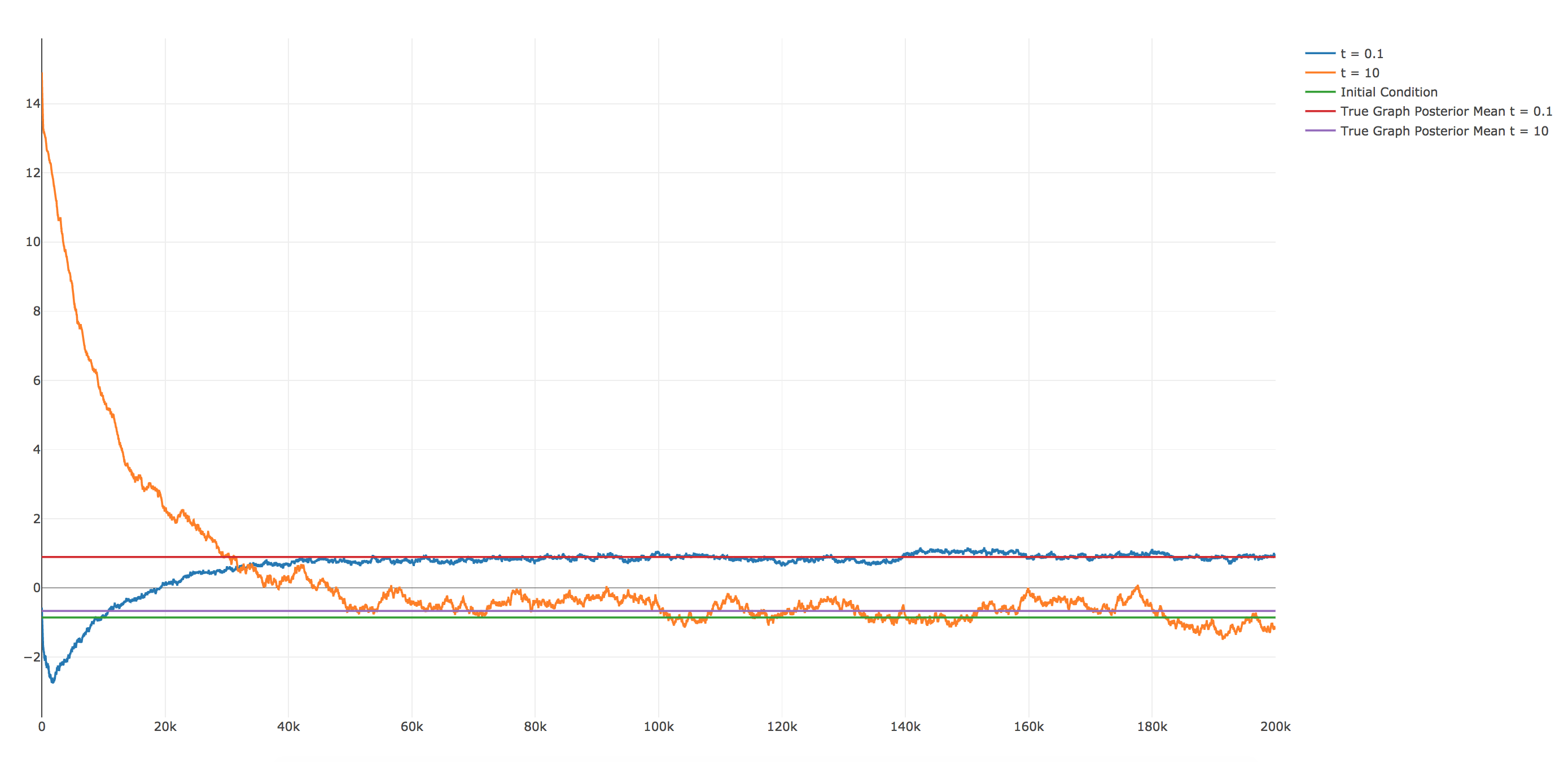}
		\caption{Shown here is the graph-pCN's chain mixing and converging for different values of the parameter $t$. Other parameter values for both chains are the same; note that the variation from $t=0.1$ to $t=10$ does not significantly affect the characteristics of the chain.} 
		\label{fig:Mixing chain different times}
	\end{figure}
	
	\begin{figure}
		\includegraphics[scale=0.35]{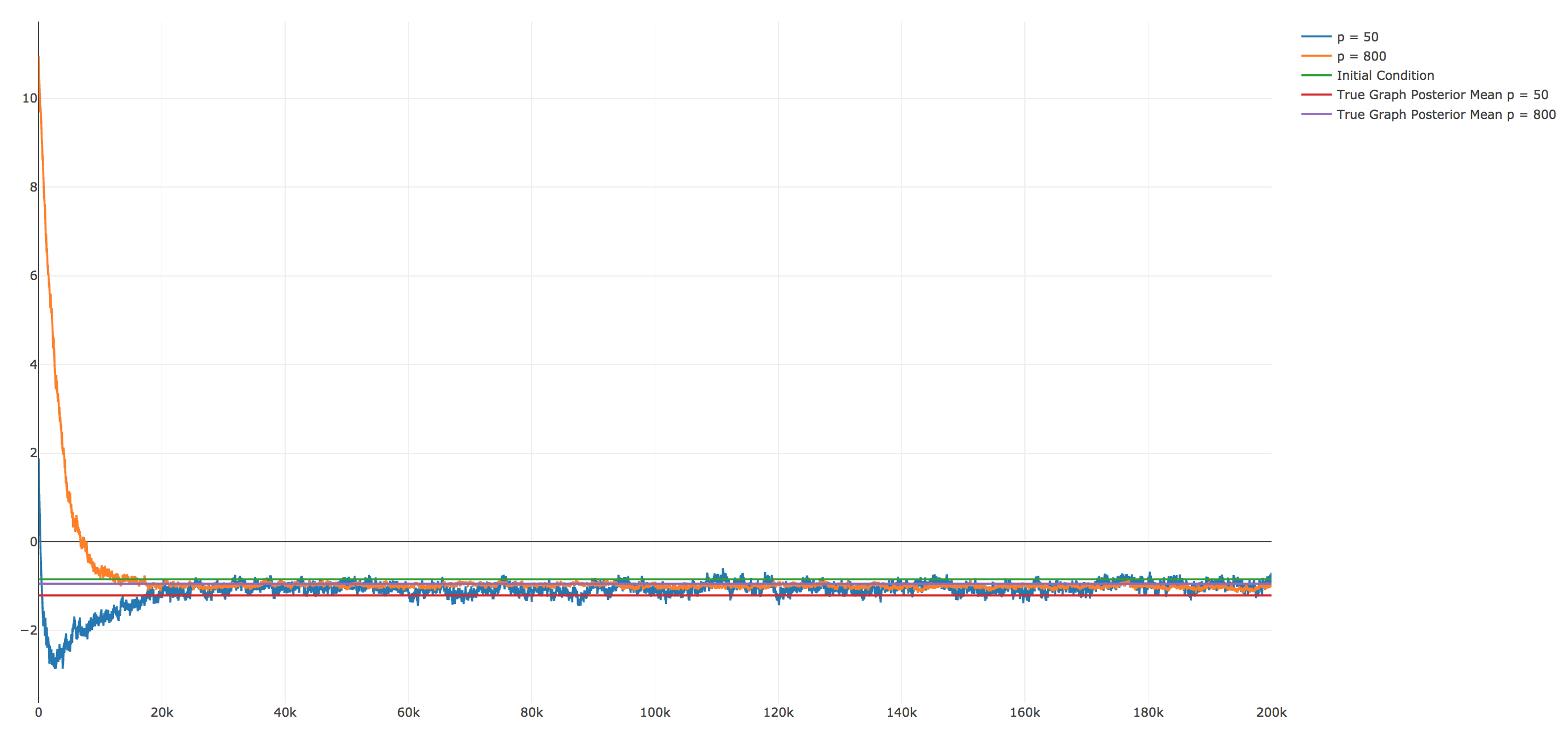}
		\caption{The above chart shows how increasing the value of the parameter $p$ reduces the variance of the chain. Again, the chains above are both from the graph-pCN algorithm, and all other parameters are chosen so that the algorithm performs optimally. } 
		\label{fig:Mixing chain different values of p}
	\end{figure}
	
	\subsection{Algorithmic Scalability} \label{ssec:pcnrobustness}
	
	It is important to stress that the large $n$ robust performance of pCN methods established in this paper hinges on the existence of a continuum limit for the measures $\mun.$ Indeed, the fact that the limit posterior $\muu$ over infinite dimensional functions can be written as a change of measure from the limit prior $\pii$ has been rigorously shown to be equivalent to the limit learning problem having \emph{finite} intrinsic dimension \citep{agapiou2015importance}. In such a case, a key principle for the robust large $n$ sampling of the measures $\mun$ is to exploit the existence of a limit density, and use some variant of the dominating measure to obtain proposal samples. It has been established ---and we do so here in the context of graph-based methods--- that careful implementation of this principle leads to robust MCMC and importance sampling methodologies \citep{hairer2014spectral,agapiou2015importance}. 
	
	A further point to note is that ---even though from a theoretical and applied viewpoint it is clearly desirable that the data is informative--- computational challenges in Bayesian settings often arise when the data is highly informative. This is also the case in the context of importance sampling and particle filters \citep{agapiou2015importance,sanz2016importance}, where certain notion of distance between prior and proposal characterizes the algorithmic complexity. In the context of the pCN MCMC algorithms, if $\Phi$ is constant, the algorithm has acceptance probability $1.$ On the other hand, large Lipschitz constant of $\Phi$ (which translates to a posterior that is far from the prior) leads to small spectral gap. Indeed, tracking the spectral gap of pCN in terms of model parameters via the understanding of Lipschitz constants is in principle possible, and will be the subject of further work. In particular, small observation noise $\sigma$ leads to deterioration of the pCN performance, see Figure \ref{fig:Mixing chain different sigmas}. This issue may be alleviated by the use of the generalized version of pCN introduced in \cite{rudolf2015generalization}. Figures \ref{fig:Mixing chain different times} and \ref{fig:Mixing chain different values of p} investigate the role of the parameters $t$ and $p$. All these figures show the posterior mean at one of the inputs, and the true graph posterior means have been computed with the formulae in the appendix.
		
	Table \ref{table3} shows the large $n$ robustness of pCN methods, while table \ref{table4} exhibits its deterioration in the fully supervised case $n=p.$ The tables show the average acceptance probability with model parameters  $\beta = 0.01,$ $p=200,$ $\varepsilon_n = 2 n^{-1/4}$ for the semi-supervised setting, and same parameters but with $p=n$ for the fully supervised case. The corresponding graph-posterior means are shown in Figure \ref{fig:Discrete Posterior Differences with Variable n}.

	\begin{figure}
		\includegraphics[scale=0.35]{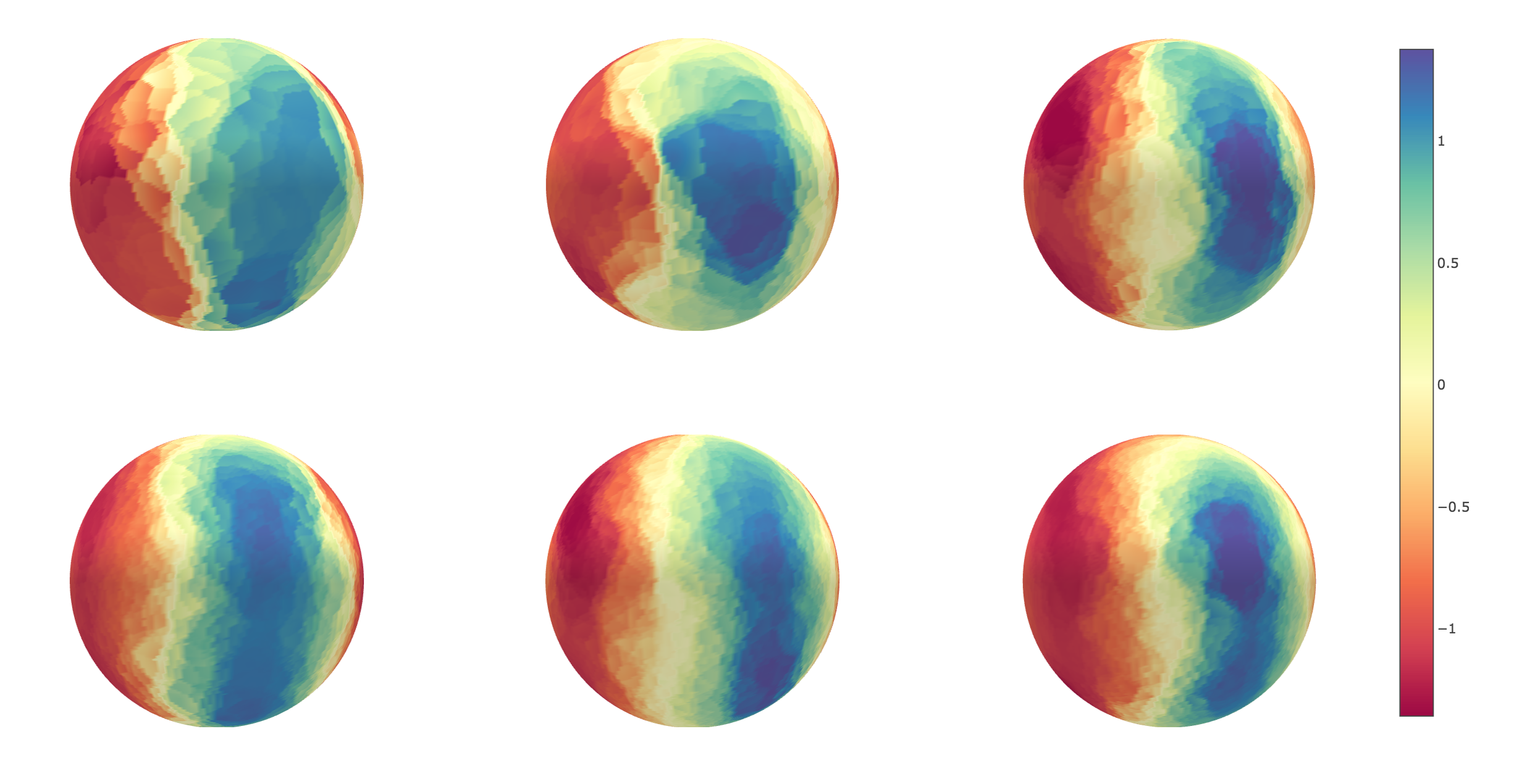}
		\caption{Graph pCN's robustness with respect to a changing value of $n$. In all plots $P_{i,j}$ above, $p = 200,\ t = 0.1,\ \sigma = 0.1,\ \text{and } \varepsilon = 2n^{-1/4}$. The plots are arranged such that $n = [300, 600, 900]$ for $P_{1,j}$ and $n=[1200, 1500, 2000]$ for $P_{2,j}$. The average acceptance probability remains constant with fixed $\beta$, as shown in Table \ref{table3}.} 
		\label{fig:Discrete Posterior Differences with Variable n}
	\end{figure}

	\begin{table}
		\caption{\label{table3} Average acceptance probability for the graph pCN in the semi-supervised setting with constant data-set of size $p=200$ and increasing number of unlabeled data.}
		\centering
		\begin{tabular}{|l|l|l|l|l|l|l|}
			\hline
			$n$                      & 300 & 600 & 900 & 1200 & 1500 & 2000 \\ \hline
			Acceptance Probability & 0.230 & 0.245 & 0.237 & 0.249  & 0.236  & 0.239  \\ \hline
		\end{tabular}
	\end{table}
	
	\begin{table}
		\caption{\label{table4} Deterioration of the average acceptance probability in a fully-supervised setting with $n=p.$ The parameter $\beta$ was held constant at $\beta = 0.01$. Additionally, $\varepsilon = 2n^{-1/4}$ and $t = 0$.}
		\centering
		\begin{tabular}{|l|l|l|l|l|l|l|}
			\hline
			$n=p$                     & 300 & 600 & 900 & 1200 & 1500 &  2000 \\ \hline
			Acceptance Probability & 0.4536 & 0.3144 & 0.2360 & 0.1924  & 0.1644  & 0.1100  \\ \hline
		\end{tabular}
	\end{table}

	\section{Acknowledgements}\label{sec:conclusions}
The work of NGT and DSA was supported by the NSF Grant DMS-1912818/1912802. ZK was funded by the NSF grant $\#1148284$ IDyaS;  TS would like to thank the Brown Division of Applied Mathematics for providing funds for the research. The authors are thankful to Dejan Slep\v{c}ev for a careful reading of a first version of this manuscript.

	\bibliography{isbib}

\begin{thebibliography}{63}
\providecommand{\natexlab}[1]{#1}
\providecommand{\url}[1]{\texttt{#1}}
\expandafter\ifx\csname urlstyle\endcsname\relax
  \providecommand{\doi}[1]{doi: #1}\else
  \providecommand{\doi}{doi: \begingroup \urlstyle{rm}\Url}\fi

\bibitem[Agapiou et~al.(2017)Agapiou, Papaspiliopoulos, Sanz-Alonso, and
  Stuart]{agapiou2015importance}
S.~Agapiou, O.~Papaspiliopoulos, D.~Sanz-Alonso, and A.~M. Stuart.
\newblock Importance sampling: Intrinsic dimension and computational cost.
\newblock \emph{Statistical Science}, 32\penalty0 (3):\penalty0 405--431, 2017.

\bibitem[Arridge et~al.(2006)Arridge, Kaipio, Kolehmainen, Schweiger,
  Somersalo, Tarvainen, and Vauhkonen]{arridge2006approximation}
S.~R. Arridge, J.~P. Kaipio, V.~Kolehmainen, M.~Schweiger, E.~Somersalo,
  T.~Tarvainen, and M.~Vauhkonen.
\newblock Approximation errors and model reduction with an application in
  optical diffusion tomography.
\newblock \emph{Inverse Problems}, 22\penalty0 (1):\penalty0 175, 2006.

\bibitem[Bardenet et~al.(2017)Bardenet, Doucet, and Holmes]{doucettalldata}
R.~Bardenet, A.~Doucet, and C.~Holmes.
\newblock {On Markov chain Monte Carlo methods for tall data}.
\newblock \emph{Journal of Machine Learning Research}, 18\penalty0
  (47):\penalty0 1--43, 2017.
\newblock URL \url{http://jmlr.org/papers/v18/15-205.html}.

\bibitem[Beaumont(2003)]{beaumont2003estimation}
M.~A. Beaumont.
\newblock Estimation of population growth or decline in genetically monitored
  populations.
\newblock \emph{Genetics}, 164\penalty0 (3):\penalty0 1139--1160, 2003.

\bibitem[Beaumont et~al.(2002)Beaumont, Zhang, and
  Balding]{beaumont2002approximate}
M.~A. Beaumont, W.~Zhang, and D.~J. Balding.
\newblock {Approximate Bayesian computation in population genetics}.
\newblock \emph{Genetics}, 162\penalty0 (4):\penalty0 2025--2035, 2002.

\bibitem[Belkin and Niyogi(2004)]{belkin2004semi}
M.~Belkin and P.~Niyogi.
\newblock {Semi-supervised learning on Riemannian manifolds}.
\newblock \emph{Machine learning}, 56\penalty0 (1-3):\penalty0 209--239, 2004.

\bibitem[Belkin and Niyogi(2005)]{belkin2005towards}
M.~Belkin and P.~Niyogi.
\newblock {Towards a theoretical foundation for Laplacian-based manifold
  methods}.
\newblock In \emph{COLT}, volume 3559, pages 486--500. Springer, 2005.

\bibitem[Belkin and Niyogi(2007)]{belkin2007convergence}
M.~Belkin and P.~Niyogi.
\newblock Convergence of {L}aplacian eigenmaps.
\newblock \emph{Advances in Neural Information Processing Systems (NIPS)},
  19:\penalty0 129, 2007.

\bibitem[Belkin and Niyogi(2008)]{bel_niy_LB}
M.~Belkin and P.~Niyogi.
\newblock Towards a theoretical foundation for {L}aplacian-based manifold
  methods.
\newblock \emph{J. Comput. System Sci.}, 74\penalty0 (8):\penalty0 1289--1308,
  2008.
\newblock ISSN 0022-0000.
\newblock \doi{10.1016/j.jcss.2007.08.006}.
\newblock URL \url{http://dx.doi.org/10.1016/j.jcss.2007.08.006}.

\bibitem[Belkin et~al.(2006)Belkin, Niyogi, and Sindhwani]{belkin2006manifold}
M.~Belkin, P.~Niyogi, and V.~Sindhwani.
\newblock {Manifold regularization: A geometric framework for learning from
  labeled and unlabeled examples}.
\newblock \emph{Journal of machine learning research}, 7\penalty0
  (Nov):\penalty0 2399--2434, 2006.

\bibitem[Bertozzi et~al.(2018)Bertozzi, Luo, Stuart, and
  Zygalakis]{bertozziuncertainty}
A.~L. Bertozzi, X.~Luo, A.~M. Stuart, and K.~C. Zygalakis.
\newblock Uncertainty quantification in graph-based classification of high
  dimensional data.
\newblock \emph{SIAM/ASA Journal on Uncertainty Quantification}, 6\penalty0
  (2):\penalty0 568--595, 2018.

\bibitem[Beskos et~al.(2008)Beskos, Roberts, Stuart, and Voss]{beskos2008mcmc}
A.~Beskos, G.~O. Roberts, A.~M. Stuart, and J.~Voss.
\newblock {MCMC methods for diffusion bridges}.
\newblock \emph{Stochastics and Dynamics}, 8\penalty0 (03):\penalty0 319--350,
  2008.

\bibitem[Blum and Chawla(2001)]{blum2001learning}
A.~Blum and S.~Chawla.
\newblock Learning from labeled and unlabeled data using graph mincuts.
\newblock 2001.

\bibitem[Burago et~al.(2014)Burago, Ivanov, and Kurylev]{burago2014graph}
D.~Burago, S.~Ivanov, and Y.~Kurylev.
\newblock {A graph discretization of the Laplace-Beltrami operator}.
\newblock \emph{J. Spectr. Theory}, 4:\penalty0 675–714, 2014.

\bibitem[Cotter et~al.(2009)Cotter, Dashti, Robinson, and
  Stuart]{cotter2009bayesian}
S.~L. Cotter, M.~Dashti, J.~C. Robinson, and A.~M. Stuart.
\newblock Bayesian inverse problems for functions and applications to fluid
  mechanics.
\newblock \emph{Inverse problems}, 25\penalty0 (11):\penalty0 115008, 2009.

\bibitem[Cotter et~al.(2013)Cotter, Roberts, Stuart, and White]{cotter2013mcmc}
S.~L. Cotter, G.~O. Roberts, A.~M. Stuart, and D.~White.
\newblock {MCMC methods for functions: modifying old algorithms to make them
  faster}.
\newblock \emph{Statistical Science}, 28\penalty0 (3):\penalty0 424--446, 2013.

\bibitem[Cui et~al.(2015)Cui, Marzouk, and Willcox]{cui2015data}
T.~Cui, Y.~M. Marzouk, and K.~E. Willcox.
\newblock {Data-driven model reduction for the Bayesian solution of inverse
  problems}.
\newblock \emph{International Journal for Numerical Methods in Engineering},
  102\penalty0 (5):\penalty0 966--990, 2015.

\bibitem[Dashti and Stuart()]{DS15}
M.~Dashti and A.~M. Stuart.
\newblock The bayesian approach to inverse problems.
\newblock Handbook of Uncertainty Quantification.

\bibitem[Donoho and Grimes(2003)]{donoho2003hessian}
D.~L. Donoho and C.~Grimes.
\newblock {Hessian eigenmaps: Locally linear embedding techniques for
  high-dimensional data}.
\newblock \emph{Proceedings of the National Academy of Sciences}, 100\penalty0
  (10):\penalty0 5591--5596, 2003.

\bibitem[El~Alaoui et~al.(2016)El~Alaoui, Cheng, Ramdas, Wainwright, and
  Jordan]{el2016asymptotic}
A.~El~Alaoui, X.~Cheng, A.~Ramdas, M.~J. Wainwright, and M.~I. Jordan.
\newblock {Asymptotic behavior of $l_p$-based Laplacian regularization in
  semi-supervised learning}.
\newblock In \emph{Conference on Learning Theory}, pages 879--906, 2016.

\bibitem[Gao et~al.(2019)Gao, Kovalsky, and Daubechies]{gao2019gaussian}
T.~Gao, S.~Z. Kovalsky, and I.~Daubechies.
\newblock Gaussian process landmarking on manifolds.
\newblock \emph{SIAM Journal on Mathematics of Data Science}, 1\penalty0
  (1):\penalty0 208--236, 2019.

\bibitem[Garc{\'\i}a~Trillos and Murray(2017)]{trillos2017new}
N.~Garc{\'\i}a~Trillos and R.~Murray.
\newblock A new analytical approach to consistency and overfitting in
  regularized empirical risk minimization.
\newblock \emph{European Journal of Applied Mathematics}, pages 1--36, 2017.

\bibitem[Garc{\'\i}a~Trillos and Sanz-Alonso(2017)]{trillos2016bayesian}
N.~Garc{\'\i}a~Trillos and D.~Sanz-Alonso.
\newblock {The Bayesian formulation and well-posedness of fractional elliptic
  inverse problems}.
\newblock \emph{Inverse Problems}, 33\penalty0 (6):\penalty0 065006, 2017.

\bibitem[Garc{\'\i}a~Trillos and
  Sanz-Alonso(2018{\natexlab{a}})]{garcia-sanz2017continuum}
N.~Garc{\'\i}a~Trillos and D.~Sanz-Alonso.
\newblock {Continuum limits of posteriors in graph Bayesian inverse problems}.
\newblock \emph{SIAM Journal on Mathematical Analysis}, 50\penalty0
  (4):\penalty0 4020--4040, 2018{\natexlab{a}}.

\bibitem[Garc{\'\i}a~Trillos and
  Sanz-Alonso(2018{\natexlab{b}})]{trillossanzflows}
N.~Garc{\'\i}a~Trillos and D.~Sanz-Alonso.
\newblock {The Bayesian update: variational formulations and gradient flows}.
\newblock \emph{Bayesian Analysis}, 2018{\natexlab{b}}.

\bibitem[Garc{\'\i}a~Trillos and Slep{\v{c}}ev(2014)]{trillos2014canadian}
N.~Garc{\'\i}a~Trillos and D.~Slep{\v{c}}ev.
\newblock On the rate of convergence of empirical measures in
  $\infty$-transportation distance.
\newblock \emph{Canadian Journal of Mathematics}, 67:\penalty0 1358--1383,
  2014.

\bibitem[Garc{\'\i}a~Trillos and Slep{\v{c}}ev(2016{\natexlab{a}})]{trillos}
N.~Garc{\'\i}a~Trillos and D.~Slep{\v{c}}ev.
\newblock Continuum limit of total variation on point clouds.
\newblock \emph{Archive for rational mechanics and analysis}, 220\penalty0
  (1):\penalty0 193--241, 2016{\natexlab{a}}.

\bibitem[Garc{\'\i}a~Trillos and
  Slep{\v{c}}ev(2016{\natexlab{b}})]{trillosACHA}
N.~Garc{\'\i}a~Trillos and D.~Slep{\v{c}}ev.
\newblock A variational approach to the consistency of spectral clustering.
\newblock \emph{Applied and Computational Harmonic Analysis},
  2016{\natexlab{b}}.

\bibitem[Garc{\'\i}a~Trillos et~al.(2018)Garc{\'\i}a~Trillos, Gerlach, Hein,
  and Slep{\v{c}}ev]{SpecRatesTrillos}
N.~Garc{\'\i}a~Trillos, M.~Gerlach, M.~Hein, and D.~Slep{\v{c}}ev.
\newblock Error estimates for spectral convergence of the graph laplacian on
  random geometric graphs towards the laplace--beltrami operator.
\newblock \emph{arXiv preprint arXiv:1801.10108}, 2018.

\bibitem[Garc{\'\i}a~Trillos et~al.(2019)Garc{\'\i}a~Trillos, Sanz-Alonso, and
  Yang]{ruiyilocalregularization}
N.~Garc{\'\i}a~Trillos, D.~Sanz-Alonso, and R.~Yang.
\newblock Local regularization of noisy point clouds: Improved global geometric
  estimates and data analysis.
\newblock \emph{Journal of Machine Learning Research}, 20\penalty0
  (136):\penalty0 1--37, 2019.
\newblock URL \url{http://jmlr.org/papers/v20/19-261.html}.

\bibitem[Geyer and Thompson(1995)]{geyer1995annealing}
C.~J. Geyer and E.~A. Thompson.
\newblock {Annealing Markov chain Monte Carlo with applications to ancestral
  inference}.
\newblock \emph{Journal of the American Statistical Association}, 90\penalty0
  (431):\penalty0 909--920, 1995.

\bibitem[Gin{\'e} and Koltchinskii(2006)]{GK}
E.~Gin{\'e} and V.~Koltchinskii.
\newblock Empirical graph {L}aplacian approximation of {L}aplace-{B}eltrami
  operators: large sample results.
\newblock In \emph{High dimensional probability}, volume~51 of \emph{IMS
  Lecture Notes Monogr. Ser.}, pages 238--259. Inst. Math. Statist., Beachwood,
  OH, 2006.
\newblock \doi{10.1214/074921706000000888}.
\newblock URL \url{http://dx.doi.org/10.1214/074921706000000888}.

\bibitem[Hairer et~al.(2011)Hairer, Mattingly, and
  Scheutzow]{hairer2011asymptotic}
M.~Hairer, J.~C. Mattingly, and M.~Scheutzow.
\newblock {Asymptotic coupling and a general form of Harris’ theorem with
  applications to stochastic delay equations}.
\newblock \emph{Probability theory and related fields}, 149\penalty0
  (1):\penalty0 223--259, 2011.

\bibitem[Hairer et~al.(2014)Hairer, Stuart, and Vollmer]{hairer2014spectral}
M.~Hairer, A.~M. Stuart, and S.~J. Vollmer.
\newblock {Spectral gaps for a Metropolis--Hastings algorithm in infinite
  dimensions}.
\newblock \emph{The Annals of Applied Probability}, 24\penalty0 (6):\penalty0
  2455--2490, 2014.

\bibitem[Harlim et~al.(2019)Harlim, Sanz-Alonso, and Yang]{harlim2019kernel}
J.~Harlim, D.~Sanz-Alonso, and R.~Yang.
\newblock {Kernel methods for Bayesian elliptic inverse problems on manifolds}.
\newblock \emph{arXiv preprint arXiv:1910.10669}, 2019.

\bibitem[Hartog and van Zanten(2016)]{hartog2016nonparametric}
J.~Hartog and H.~van Zanten.
\newblock {Nonparametric Bayesian label prediction on a graph}.
\newblock \emph{arXiv preprint arXiv:1612.01930}, 2016.

\bibitem[Hein(2006)]{Hei2006}
M.~Hein.
\newblock Uniform convergence of adaptive graph-based regularization.
\newblock In G.~Lugosi and H.~U. Simon, editors, \emph{Proc. of the 19th Annual
  Conference on Learning Theory (COLT)}, pages 50--64. Springer, 2006.

\bibitem[Hein et~al.(2007)Hein, Audibert, and Von~Luxburg]{HeAuvL07}
M.~Hein, J-Y Audibert, and U.~Von~Luxburg.
\newblock Graph laplacians and their convergence on random neighborhood graphs.
\newblock \emph{Journal of Machine Learning Research}, 8\penalty0
  (Jun):\penalty0 1325--1368, 2007.

\bibitem[Kennedy and O'Hagan(2001)]{kennedy2001bayesian}
M.~C. Kennedy and A.~O'Hagan.
\newblock Bayesian calibration of computer models.
\newblock \emph{Journal of the Royal Statistical Society: Series B (Statistical
  Methodology)}, 63\penalty0 (3):\penalty0 425--464, 2001.

\bibitem[Kipnis and Varadhan(1986)]{kipnis1986central}
C.~Kipnis and S.~R.~S. Varadhan.
\newblock {Central limit theorem for additive functionals of reversible Markov
  processes and applications to simple exclusions}.
\newblock \emph{Communications in Mathematical Physics}, 104\penalty0
  (1):\penalty0 1--19, 1986.

\bibitem[Liang et~al.(2007)Liang, Mukherjee, and West]{sayan}
F.~Liang, S.~Mukherjee, and M.~West.
\newblock The use of unlabeled data in predictive modeling.
\newblock \emph{Statistical Science}, pages 189--205, 2007.

\bibitem[Maier et~al.(2009)Maier, Von~Luxburg, and Hein]{maier2009influence}
M.~Maier, U.~Von~Luxburg, and M.~Hein.
\newblock Influence of graph construction on graph-based clustering measures.
\newblock In \emph{Advances in neural information processing systems}, pages
  1025--1032, 2009.

\bibitem[Marzouk et~al.(2007)Marzouk, Najm, and Rahn]{marzouk2007stochastic}
Y.~M. Marzouk, H.~N. Najm, and L.~A. Rahn.
\newblock {Stochastic spectral methods for efficient Bayesian solution of
  inverse problems}.
\newblock \emph{Journal of Computational Physics}, 224\penalty0 (2):\penalty0
  560--586, 2007.

\bibitem[Olver(2013)]{olver2013introduction}
P.~J. Olver.
\newblock \emph{Introduction to Partial Differential Equations}.
\newblock Springer Science \& Business Media, 2013.

\bibitem[Rasmussen and Williams(2006)]{rasmussen2006gaussian}
C.~E. Rasmussen and C.~K.~I. Williams.
\newblock \emph{{Gaussian Processes for Machine Learning}}, volume~1.
\newblock MIT press Cambridge, 2006.

\bibitem[Roweis and Saul(2000)]{roweis2000nonlinear}
S.~T. Roweis and L.~K. Saul.
\newblock Nonlinear dimensionality reduction by locally linear embedding.
\newblock \emph{{Science}}, 290\penalty0 (5500):\penalty0 2323--2326, 2000.

\bibitem[Rudolf and Sprungk(2015)]{rudolf2015generalization}
D.~Rudolf and B.~Sprungk.
\newblock {On a generalization of the Preconditioned Crank--Nicolson Metropolis
  algorithm}.
\newblock \emph{Foundations of Computational Mathematics}, pages 1--35, 2015.

\bibitem[Sacks et~al.(1989)Sacks, Welch, Mitchell, and Wynn]{sacks1989design}
J.~Sacks, W.~J. Welch, T.~J. Mitchell, and H.~P. Wynn.
\newblock Design and analysis of computer experiments.
\newblock \emph{Statistical science}, pages 409--423, 1989.

\bibitem[Sanz-Alonso(2018)]{sanz2016importance}
D.~Sanz-Alonso.
\newblock Importance sampling and necessary sample size: an information theory
  approach.
\newblock \emph{SIAM/ASA Journal on Uncertainty Quantification}, 6\penalty0
  (2):\penalty0 867--879, 2018.

\bibitem[Shi(2015)]{Shi2015}
Z.~Shi.
\newblock {Convergence of Laplacian spectra from random samples}.
\newblock arXiv preprint arXiv:1507.00151, 2015.

\bibitem[Sindhwani et~al.(2005)Sindhwani, Niyogi, and
  Belkin]{sindhwani2005beyond}
V.~Sindhwani, P.~Niyogi, and M.~Belkin.
\newblock Beyond the point cloud: from transductive to semi-supervised
  learning.
\newblock In \emph{Proceedings of the 22nd international conference on Machine
  learning}, pages 824--831. ACM, 2005.

\bibitem[Singer(2006)]{singer06}
A.~Singer.
\newblock From graph to manifold {L}aplacian: The convergence rate.
\newblock \emph{Applied and Computational Harmonic Analysis}, 21\penalty0
  (1):\penalty0 128--134, 2006.

\bibitem[Singer and Wu(2017)]{SinWu13}
A.~Singer and H-T Wu.
\newblock {Spectral convergence of the connection Laplacian from random
  samples}.
\newblock \emph{Information and Inference: A Journal of the IMA}, 6\penalty0
  (1):\penalty0 58--123, 2017.

\bibitem[Slep{\v{c}}ev and Thorpe(2017)]{slepvcev2017analysis}
D.~Slep{\v{c}}ev and M.~Thorpe.
\newblock {Analysis of p-Laplacian regularization in semi-supervised learning}.
\newblock \emph{arXiv preprint arXiv:1707.06213}, 2017.

\bibitem[Stein(2012)]{stein2012interpolation}
M.~L. Stein.
\newblock \emph{Interpolation of Spatial Data: Some Theory for Kriging}.
\newblock Springer Science \& Business Media, 2012.

\bibitem[Stuart and Teckentrup(2017)]{stuart2017posterior}
A.~M. Stuart and A.~Teckentrup.
\newblock {Posterior consistency for Gaussian process approximations of
  Bayesian posterior distributions}.
\newblock \emph{Mathematics of Computation}, 2017.

\bibitem[Tenenbaum et~al.(2000)Tenenbaum, De~Silva, and
  Langford]{tenenbaum2000global}
J.~B. Tenenbaum, V.~De~Silva, and J.~C. Langford.
\newblock A global geometric framework for nonlinear dimensionality reduction.
\newblock \emph{{Science}}, 290\penalty0 (5500):\penalty0 2319--2323, 2000.

\bibitem[Ting et~al.(2010)Ting, Huang, and Jordan]{THJ}
D.~Ting, L.~Huang, and M.~I. Jordan.
\newblock An analysis of the convergence of graph {L}aplacians.
\newblock In \emph{Proc. of the 27th Int. Conference on Machine Learning
  (ICML)}, 2010.

\bibitem[Von~Luxburg(2007)]{von2007tutorial}
U.~Von~Luxburg.
\newblock A tutorial on spectral clustering.
\newblock \emph{Statistics and Computing}, 17\penalty0 (4):\penalty0 395--416,
  2007.

\bibitem[Xiu(2010)]{xiu2010numerical}
D.~Xiu.
\newblock \emph{Numerical methods for stochastic computations: a spectral
  method approach}.
\newblock Princeton University Press, 2010.

\bibitem[Zhou and Sch{\"o}lkopf(2005)]{zhou2005regularization}
D.~Zhou and B.~Sch{\"o}lkopf.
\newblock Regularization on discrete spaces.
\newblock In \emph{Joint Pattern Recognition Symposium}, pages 361--368.
  Springer, 2005.

\bibitem[Zhu(2005)]{zhu2005semi}
X.~Zhu.
\newblock Semi-supervised learning literature survey.
\newblock 2005.

\bibitem[Zhu et~al.(2003)Zhu, Ghahramani, and Lafferty]{zhu2003semi}
X.~Zhu, Z.~Ghahramani, and J.~D. Lafferty.
\newblock {Semi-supervised learning using Gaussian fields and harmonic
  functions}.
\newblock In \emph{Proceedings of the 20th International conference on Machine
  learning (ICML-03)}, pages 912--919, 2003.

\end{thebibliography}
	
	\appendix

		\section{Benchmark Formulae}\label{ssec:benchmarkformulae}
	Here we exploit the linearity of the forward and observation maps to compute, under the Gaussian observation noise model, the mean and covariance of the Gaussian graph and continuum posteriors. These formulae could be useful in understanding the approximation of continuum posteriors by graph posteriors, and to provide benchmarks for posteriors computed with MCMC methods. For the derivations we use the covariance function representation of Gaussian measures and the theory of Gaussian process regression in \cite{rasmussen2006gaussian}. Throughout we assume that $s$ is large enough so that the formulae below are well-defined.
	
	We start with the continuum case. Set $v:= \F u.$ The prior \eqref{eq:prior} on $u$ induces a prior on $v\sim GP\bigl(0,c_v(x,\tilde{x})\bigr),$ where
	\begin{equation}\label{eq:covariancecv}
	c_v(x,\tilde{x}) = \sum_{i=1}^\infty e^{-2\lambda_i t} (\alpha + \lambda_i)^{-s/2} \psi_i(x) \psi_i(\tilde{x}).
	\end{equation}
	Then, we have a regression problem for $v$ given data $y = [y_1,\ldots,y_p]'$
	$$y_i = v(\x_i) + \eta_i, \quad \eta_i \sim N(0,\gamma^2)$$
	in the form of \cite{rasmussen2006gaussian}. The posterior distribution of $v|y$ is thus given by a Gaussian process $GP\bigl(m_{v|y}(x), c_{v|y}(x,\tilde{x})\bigr),$ with
	\begin{align*}
	m_{v|y}(x) &= c_v(x,X)' \bigl(c_v(X,X) + \gamma^2I\bigr)^{-1} \,y,\\
	c_{v|y} (x,\tilde{x}) &= c_v(x,\tilde{x}) - c_v(x,X)' \bigl(c_v(X,X) + \gamma^2I \bigr)^{-1} c_v(\tilde{x},X),
	\end{align*}
	where we use the following notations:
	\begin{align*}
	c_v(x,X) &:= [c_v(x,\x_1), \ldots, c_v(x,\x_p) ]' \in \R^p, \\
	c_v(X,X) &:=\bigl(c_v(\x_i,\x_j)\bigr)_{1\le i,j\le p} \in \R^{p\times p}.
	\end{align*}
	Now the posterior of interest $\muu$ on $u$ given $y$ can be recovered by running the heat equation backwards. Namely, we have that $\muu = GP\bigl(m_{u|y}(x), c_{u|y}(x,\tilde{x})\bigr)$ with
	\begin{align}\label{eq:meancovariancecont}
	\begin{split}
	m_{u|y}(x) &= c_w(x,X)' (c_v(X,X)  + \gamma^2I )^{-1} y, \\
	c_{u|y}(x,\tilde{x}) &=c_u(x,\tilde{x}) - c_w(x,X)' \bigl(c_v(X,X) + \gamma^2I\bigr)^{-1} c_w(\tilde{x},X),
	\end{split}
	\end{align}
	where $c_w(x,X)$ is a vector made of evaluations of the covariance function of $w:=\F^{1/2}u$ at the test and training points. Precisely, its $j$-th entry is given by
	\begin{equation}\label{eq:covariancecw}
	c_w(x,X)_j =  \sum_{i=1}^\infty e^{-\lambda_it} (\alpha + \lambda_i)^{-s/2}  \psi_i(x) \psi_i(\x_j).
	\end{equation}
	
	There are several points to note about equation \eqref{eq:meancovariancecont}. First, the predictive mean is a linear function of the data $y$, hence a linear predictor. It is indeed the best linear predictor in a mean-squared error sense \citep{stein2012interpolation}. Second, since $c_v(X,X) + \gamma^2I$ is positive definite, $c_{u|y}(x,\tilde{x}) \le c_{u}(x,\tilde{x})$; thus, conditioning reduces the uncertainty. Moreover, in the limit of noiseless observations ($\gamma = 0$) and $t=0$ we recover that $c_{u|y}(\x_i,\x_j) = 0$ in the training points. However, even with noiseless observations this is not true if $t>0.$ Finally, note the well-known fact that the the posterior covariance $c_{u|y}$ does not depend on the observed data $y$.
	
	Formulae in the discrete setting can be obtained in a similar way, and we omit the details. Plugging in the data $y$ from the continuum setting, we deduce that $$\mun = N\bigl( m_{u_n|y}(\x_k), c_{u_n|y}(\x_k, \x_l)\bigr),$$ with
	
	\begin{align}\label{eq:meancovariancedisc}
	\begin{split}
	m_{u_n|y}(\x_k) &= c_{w_n}(x,X)' (c_{v_n}(X,X)  + \gamma^2I )^{-1} y, \\
	c_{u_n|y}(\x_k,\x_l) &=c_{u_n}(\x_k,\x_l) - c_{w_n}(\x_k,X)' \bigl(c_{v_n}(X,X) + \gamma^2I\bigr)^{-1} c_{w_n}(\x_l,X).
	\end{split}
	\end{align}
	
	In the above equations, all objects indexed by $n$ constitute straightforward analogues of objects in the continuum, constructed using the graph spectrum rather than the continuum one.

	\section{The $TL^2$ and $\mathcal{P}(TL^2)$ Spaces}
	\label{TL2PTL2}
	
	Let us recall the definition of the $TL^2$ space. First, we define the set
	\[ TL^2 := \bigl\{ (\theta, f) \; : \:  \theta \in \mathcal P(\M), \, f \in L^p(\M, \theta) \bigr\}. \]
	Then, for arbitrary elements $(\theta_1,f_1)$ and $(\theta_2,f_2)$ in $TL^2$ we define, following \cite{trillos},  
	\begin{align} \label{tlpmetric}
	\begin{split}
	d_{TL^2}\bigl((\theta_1,f_1), (\theta_2,f_2)\bigr):=
	\inf_{\omega \in \Gamma(\theta_1, \theta_2)} \left(  \iint_{\M \times \M} \Bigl( d_\M(x,y)^2 + |f_1(x)-f_2(y)|^2  \Bigr) d\omega(x,y)  \right)^{1/2},
	\end{split}
	\end{align}
	where $\Gamma(\theta_1,\theta_2)$ is the set of Borel probability measures on $\M\times \M$ with marginal $\theta_1$ on the first factor and $\theta_2$ on the second one. It was shown in \cite{trillos} that $d_{TL^2}$ defines a distance in $TL^2$.
	
	The $TL^2$ space allows us to make sense of a sequence $u_n \in L^2(\gamma_n)$ converging towards an element $u \in L^2(\gamma)$. Indeed, with a slight abuse of notation, we say that a sequence $u_n \in L^2(\gamma_n)$ converges in $TL^2$ towards $u\in L^2(\gamma),$ written
	$$u_n \converges{TL^2} u,$$
	if $d_{TL^2}\bigl((u_n, \gamma_n), (u, \gamma)\bigr) \to 0.$  A characterization of convergence in $TL^2$ in terms of composition with transport maps can be found in Proposition 3.12 in \cite{trillos}. 
	
	As noted in \cite{trillos}, $(TL^2, d_{TL^2})$ is not a complete metric space. Its completion however, denoted $\overline{TL^2}$, can be identified with the space $\mathcal{P}_2(\M \times \R)$ of Borel probability measures on the product space $\M \times \R$ with finite second moments, endowed with the Wasserstein distance. The space $\overline{TL^2}$ is a Polish space.

	Having introduced the metric space $TL^2$ we can now define $\mathcal{P}(TL^2)$ to be the space of Borel probability measures on $TL^2$ endowed with the weak convergence of probability measures. If $\theta \in \mathcal{P}(\M)$ and  $\boldsymbol{\nu} \in \mathcal{P}(L^2(\theta))$, it is possible to think of $\boldsymbol{\nu}$ as an element in $\mathcal{P}(TL^2)$. Indeed, the canonical inclusion
	\[ \mathcal{I}_\theta:  f \in  L^2(\theta) \longmapsto   (\theta, f) \in TL^2 \]
	induces the canonical inclusion
	\[ \mathcal{I}_{\theta \sharp}:  \mathcal{P}(L^2(\theta)) \hookrightarrow \mathcal{P}(TL^2), \]
	where $\mathcal{I}_{\theta \sharp}$ is the push-forward via $\mathcal{I}_{\theta}$.  Notice that $\mathcal{I}_\theta$ is a continuous map. In the sequel we may drop the explicit mention to $\mathcal{I}$ whenever no confusion arises from doing so.
	
	The above observation motivates the following definition. 
	\begin{definition}
		\label{convmeasures:defn}
		For  $\boldsymbol{\nu_n} \in \mathcal{P}\bigl(L^2(\gamma_n)\bigr),$ $n\in \N,$ and   $\boldsymbol{\nu} \in \mathcal{P}\bigl(L^2(\gamma)\bigr)$ we say that $\{\boldsymbol{\nu_n}\}_{n\in \N}$ converges to $\boldsymbol{\nu}$, written
		$$\boldsymbol{\nu_n}\converges{\P(TL^2)}{ \boldsymbol{\nu}},$$
		if $\{\mathcal{I}_{\gamma _n \sharp} \boldsymbol{\nu_n}\}_{n\in \N}$ converges weakly to $ \mathcal{I}_{\gamma \sharp}  \boldsymbol{\nu}$ in $\mathcal{P}(TL^2).$
	\end{definition}
	
This is the notion of convergence of discrete to continuum posteriors that we use in this paper. The space $\mathcal{P}(TL^2)$ was introduced in \cite{garcia-sanz2017continuum}.

	\section{Proof of Theorem \ref{th:interpolant}}
	\label{AppA}
	We want to show that
	\begin{equation}
	\I_{n\sharp} \mun \rightarrow_{\mathcal{P}(L^2(\gamma))} \muu, \quad \text{ as } n \rightarrow \infty.
	\label{conv:Goal}
	\end{equation}

	\textbf{Step 0:} 
	The proof of Theorem 4.1 in \cite{garcia-sanz2017continuum} shows that   
	\[\pii_n \rightarrow_{\mathcal{P}(TL^2)} \pii, \quad \text{ as } n \rightarrow \infty,\]
	under the assumptions of Theorem \ref{th:interpolant} (in particular removing the upper bound assumption on $\veps_n$ from Theorems 4.1 and 4.4 in \cite{garcia-sanz2017continuum}). Likewise the proof of Theorem 4.4  in \cite{garcia-sanz2017continuum} establishes the $\Gamma$-convergence of the energies 
	\[ J_n(\boldsymbol{\nu_n}):=  \dkl(\boldsymbol{\nu_n} \|  \pin ) + \int_{L^2(\gamma_n)} \phi_n(u_n; y) d \boldsymbol{\nu_n}(u_n), \quad \mun \in \mathcal{P}(L^2(\gamma_n)),  \]
	towards the energy
	\[ J(\boldsymbol{\nu}) = \dkl(\boldsymbol{\nu} \| \pii ) + \int_{L^2(\gamma)} \phi(u; y) d \boldsymbol{\nu}(u), \quad \boldsymbol{\nu} \in \mathcal{P}(L^2(\gamma)) \]
	in the $\mathcal{P}(TL^2)$-sense, under the assumptions of Theorem \ref{th:interpolant}. In particular, 
	\[ \mun \rightarrow_{\mathcal{P}(TL^2)} \muu , \quad n \rightarrow \infty,\]
	because $\mun$ is the minimizer of $J_n$ and $\muu$ is the minimizer of $J$ (see the variational characterization of posterior distributions in \cite{trillossanzflows}).

	\textbf{Step 1:} We claim that $\{ \I_{n\sharp} \mun \}_{n \in \N}$ is pre-compact with respect to the weak convergence of probability measures on $L^2(\gamma)$.  By Lemma 5.1 in \cite{garcia-sanz2017continuum} it is enough to show that
	\begin{enumerate}[(i)]
		\item $\sup_{n \in \N}  \dkl( \I_{n \sharp} \mun  \| \I_{n \sharp} \pin)   < +\infty$; and
		\item $\mathcal{I}_{n \sharp} \pin \rightarrow _{\mathcal{P}(L^2(\gamma))} \pii.$
	\end{enumerate}
	
	Let us start with (i). Step 0 implies that 
	\[ \lim_{n \rightarrow \infty} \min_{\boldsymbol{\nu_n}}  J_n(\boldsymbol{\nu_n}) =   \min_{\boldsymbol{\nu}}  J(\boldsymbol{\nu}) <  +\infty. \]
	Given that $\mun$ is the minimizer of $J_n$ and $\muu$ is the minimizer of $J$, it follows that
	\[ \lim_{n \rightarrow \infty} J_n(\mun)  = J(\muu) < +\infty.\]
	Combining the previous fact with the chain of inequalities
	\[ \dkl(\I_{n\sharp} \mun \|  \I_{n\sharp} \pin) \leq  \dkl(\mun \|  \pin ) \leq J_n(\mun)  \] 
	gives (i).

	We now show (ii). Consider an orthonormal basis of eigenvectors $\{\psi_1^n , \dots, \psi_i^n \}$ of $\Delta_{\M_n}$ and an orthonormal basis $\{ \psi_1, \dots, \psi_n , \dots \}$ of eigenfunctions of $\Delta_\M$. By the results in \cite{trillosACHA} we can assume without the loss of generality that, for all $j\in \N,$ 
	\[  \psi_j^n \rightarrow_{TL^2} \psi_j , \text{ as } n \rightarrow \infty. \]
 Let $(\tilde{\Omega}, \tilde{F}, \tilde{\mathbb{P}})$ be a probability space supporting i.i.d. random variables $\{ \xi_i \}_{i \in \N}$ with $\xi_i \sim N(0,1)$ and consider 
	\[ X_n = \sum_{i=1}^{k_n}   (\alpha + \lambda_i^n)^{-s/4}\xi_i   \psi_i^n , \quad X = \sum_{i=1}^{\infty}    (\alpha + \lambda_i)^{-s/4}\xi_i   \psi_i ,\]
	 where, recall, $k_n$ is the truncation level of the prior $\pin.$ \nc
	Notice that $X_n \sim \pin $, $X \sim \pii$ and $\I_n(X_n)$ is distributed according to  $\I_{n \sharp} \pin$. For any fixed $i=1, \dots, k_n$ it follows from the first part of the proof of Theorem 1.10 in \cite{SpecRatesTrillos} that
	\begin{equation}\label{eq:bounds}
	\lVert \I_n (\psi_i^n) \rVert_{L^2(\gamma)}  \leq \lVert \I_n (\psi_i^n) - \psi_i \rVert_{L^2(\gamma)} + \lVert \psi_i \rVert_{L^2(\gamma)} \leq   C ,
	\end{equation}
	where $C$ is a constant independent of $i=1, \dots, k_n$ and $n$. It then follows that for every $l \in \N$, 
	\begin{align*}
	\lVert \I_n(X_n) & - X  \rVert_{L^2(\gamma)}  \leq \left\lVert \sum_{i=1}^l (\alpha + \lambda_i^n)^{-s/4}\xi_i \I_n (\psi_i^n) - \sum_{i=1}^l (\alpha + \lambda_i)^{-s/4}\xi_i \psi_i \right\rVert_{L^2(\gamma)} 
	\\&+ \sum_{i=l}^{k_n} (\alpha + \lambda_i^n)^{-s/4}|\xi_i| \lVert \I_n (\psi_i^n) \rVert_{L^2(\gamma)}    + \sum_{i=l}^\infty (\alpha + \lambda_i)^{-s/4} |\xi_i| \lVert \psi_i \rVert_{L^2(\gamma)}
	\\ & \leq \left\lVert \sum_{i=1}^l (\alpha + \lambda_i^n)^{-s/4}\xi_i \I_n(\psi_i^n) - \sum_{i=1}^l (\alpha + \lambda_i)^{-s/4}\xi_i \psi_i \right\rVert_{L^2(\gamma)} +  C \sum_{i=l}^\infty (\alpha + \lambda_i)^{-s/4} |\xi_i|  ,
	\end{align*}
	where $C$ is a constant that does not depend on $n$; we have used the bounds \eqref{eq:bounds} on $\lVert \I_n(\psi_i^n) \rVert_{L^2(\gamma)}$ and the bounds \eqref{eq:estimateseigen} for $\lambda_i^n$ in terms of $\lambda_i$ for $i=1, \dots, k_n$. We can then take expectations and $\limsup$s in both sides of the above inequality and use Theorem 1.10 in \cite{SpecRatesTrillos} to conclude that 
	\[ \limsup_{n \rightarrow \infty}  \E \left( \lVert \I_n(X_n) - X  \rVert_{L^2(\gamma)} \right) \leq C \sum_{i=l}^\infty (\alpha + \lambda_i)^{-s/4}. \]
	Since the above is true for every $l$ and the series is convergent, (ii) follows.
	
	An application of Lemma 5.1 in \cite{garcia-sanz2017continuum} allows us to deduce that $\{  \mathcal{I}_{n \sharp} \mun   \}_{n \in \N} \subseteq \mathcal{P}(L^2(\gamma)) $ is pre-compact and, moreover, that each of its cluster points is a measure that is absolutely continuous with respect to $\pii$. We can then assume without the loss of generality that, for some $\tilde{\muu} \in \mathcal{P}(L^2(\gamma)),$
	\[  \I_{n\sharp} \mun \rightarrow_{\mathcal{P}(L^2(\gamma))} \tilde{\muu}, \quad \text{ as } n \rightarrow \infty.\]

	\textbf{Step 2:} To show \eqref{conv:Goal} it is then enough to prove that the finite dimensional projections of $\tilde{\muu}$ coincide with those of $\muu$. More precisely, we identify $u \in L^2(\gamma)$ with the infinite vector $(u_1, u_2 , \dots)$ denoting the coefficients of $u$ in the basis $\{ \psi_1, \psi_2, \dots\}$ and define $\Proj_j(u):= \sum_{i=1}^j  u_i  \psi_i$; we need to show that for arbitrary $j\in N$ we have
	\[  \Proj_{j \sharp} \tilde{\muu}= \Proj_{j\sharp } \muu.\] 
	From Step 0 and Skorohod's theorem,  we know there exists a probability space $(\tilde{\Omega}, \tilde{F}, \tilde{\mathbb{P}})$ supporting random variables $\{ X_n^y\}_{n \in \N}$ and $X^y$ with $X_n^y \sim \mun$ and $X^y\sim \muu$ and for which $X_n^y \rightarrow_{TL^2} X^y$ for $\tilde{\Prob}$-a.e. $\tilde{\omega} \in \tilde{\Omega}$. We can then write
	\[ X_n^y= \sum_{i=1}^{k_n}   a_i^n   \psi_i^n , \quad X^y= \sum_{i=1}^{\infty}    a_i  \psi_i ,\]
	for some random variables $a_i^n$ and $a_i$. Notice that the continuity of inner products with respect to $TL^2$-convergence (see Proposition 2.6  in \cite{trillosACHA}) implies that
	\[  \lim_{n \rightarrow \infty} a_i^n =  a_i, \quad \tilde{\Prob}\text{-a.e}.   \]
	
	Now, for every fixed $l \geq j$ we can write 
	\begin{equation}
	\label{aux:appen1}
	\Proj_j( \I_n(X_n^y))  = \sum_{i=1}^l a_i^n \Proj_j( \I_n(\psi_i^n) ) +  \sum_{i=l+1}^{k_n} a_i^n \Proj_j( \I_n(\psi_i^n) ).
	\end{equation}
	The left hand side of the above expression is seen to converge weakly towards $\Proj_{j \sharp}\tilde{\muu}$ because $\I_n(X_n^y) \sim \I_{n \sharp} \mun $, $\I_{n \sharp}\mun \rightarrow_{\mathcal{P}(L^2(\gamma))} \tilde{\muu}$, and because  $\Proj_j$ is continuous. On the other hand, the first term on the right hand side is seen to converge  $\tilde{\Prob}$-a.e. towards $\sum_{i=1}^j a_i \Proj_j(\psi_i)= \sum_{i=1}^j a_i \psi_i$ because 
	\[  \I_n(\psi_i^n ) \rightarrow_{L^2(\gamma)} \psi_i , \quad \text{ as } n \rightarrow \infty,  \]
	which follows from Theorem 1.10 in \cite{SpecRatesTrillos} (it is at this stage that we need the extra technical condition on $\veps_n$); in particular this term converges weakly towards $\Proj_{j \sharp} \muu$. To show $\Proj_{j \sharp} \tilde{\muu}= \Proj_{j\sharp } \muu$ it is then enough, by Slutsky's theorem, to prove that $\lVert \sum_{i=l+1}^{k_n} a_i^n \Proj_j( \I_n(\psi_i^n) ) \rVert_{L^2(\gamma)}$ converges in probability towards zero. 
	
	To see this, first notice that
	\[  \left\lVert \sum_{i=l+1}^{k_n} a_i^n \Proj_j(\I_n (\psi_i^n)) \right\rVert_{L^2(\gamma)}  \leq C \sum_{i=l+1}^{k_n} \lvert  a_i^n\rvert.\] 
	Fix $t>0$. Observe that the expression 
	\[ \limsup_{n \rightarrow \infty}  \tilde{\Prob}\left( \left \lVert \sum_{i=l+1}^{k_n} a_i^n \Proj_j(\I_n (\psi_i^n))  \right \rVert_{L^2(\gamma)}  > t \right)\]
	is independent of $l$.  Then,
	\begin{align*}
	q_j(t) &:= \limsup_{n \rightarrow \infty}  \tilde{\Prob}\left( \left \lVert \sum_{i=l+1}^{k_n} a_i^n \Proj_j(\I_n (\psi_i^n))  \right \rVert_{L^2(\gamma)}  > t \right)  \\
	& \leq \limsup_{n \rightarrow \infty} \tilde{\Prob}\left( \sum_{i=l+1}^{k_n}|a_i^n| > \frac{t}{C}  \right).
	\end{align*}
	On the other hand, identifying the elements in the support of $\pin$ with $\R^{k_n}$ (i.e. writing $u_n \in \supp(\pin)$ in the basis $\{ \psi_1^n , \dots, \psi_{k_n}^n\}$) and letting $A_{n,t,l}$ be the set
	\[  A_{n,t,l}:=  \left\{  x\in \R^{k_n} \: : \:  \sum_{i=l+1}^{k_n}| x_i| > \frac{t}{C} \right\}, \]
	we see that
	\[ \tilde{\Prob}\left( \sum_{i=l+1}^{k_n}|a_i^n| > \frac{t}{C}  \right)  = \mun \left(A_{n,t,l} \right) = \frac{1}{Z_n}\int_{A_{n,t,l}} \exp(- \Phi_n(x;y))d \pin(x) \leq \frac{1}{Z_n}\pin\left( A_{n,t,l} \right), \]
	and hence
	\[ \limsup_{n \rightarrow \infty} \tilde{\Prob}\left( \sum_{i=l+1}^{k_n}|a_i^n| > \frac{t}{C}  \right) \leq \frac{1}{Z} \pii \left( \{ u \in L^2(\gamma) \: : \: \sum_{i=l+1}^\infty |u_i| > t/C   \}\right). \]
	In the above $Z$ and $Z_n$ are the normalization constants from \eqref{eq:posteriorknownm} and \eqref{eq:posteriorunknownm} respectively. 
	
	Therefore, 
	\[ q_j(t) \leq  \frac{1}{Z} \pii \left( \{ u \in L^2(\gamma) \: : \: \sum_{i=l+1}^\infty |u_i| > t/C   \}\right). \]
	Taking now the limit as $l \rightarrow \infty$ of the right hand side of the above expression, we deduce that $q_j(t) =0$. Since this is true for arbitrary $t>0$, we deduce that indeed $\lVert \sum_{i=l+1}^{k_n} a_i^n \Proj_j( \I_n(\psi_i^n) ) \rVert_{L^2(\gamma)}$ converges in probability towards zero and the proof is now complete.
	
	\begin{remark}
		\label{GeneralInterpolants}
		In the above proof we have used results from \cite{SpecRatesTrillos} on Voronoi extensions, but it is clear that analogue results can be deduced for more general interpolation maps $\{ \mathcal{I}_n\}_{n \in \N}$ as long as one can show the following:
		\begin{enumerate}
			\item (Uniform $L^2$-boundedness) There is a constant $C>0$ such that $\lVert \mathcal{I}_n \psi_i^n \rVert_{L^2(\gamma)} \leq C$ for every $i=1, \dots, k_n$ and for every $n$.
			\item (Consistency) For every $i \in \N$ we have $\I_n(\psi_i^n) \rightarrow_{L^2(\gamma)} \psi_i $.
		\end{enumerate}
	\end{remark}
	
	\section{Proof of Theorem \ref{th:pcn}}
	\label{AppB}

	The proof of Theorem \ref{th:pcn} is based on the paper \cite{hairer2014spectral} which in turn makes use of the following weak form of Harris theorem from \cite{hairer2011asymptotic}. We let $\H$ be a separable Hilbert space and for a distance like function $\tilde{d}: \H \times \H \rightarrow [0,\infty)$ define the associated Wasserstein distance (1-OT distance) on $\mathcal{P}(\H)$ 
	\begin{equation}
	\label{Wass}
	\tilde{d}(\mu, \nu): = \inf_{ \theta \in \Gamma(\mu, \nu) } \int_{\H \times \H} \tilde{d}(u,w) d \theta(u,w), \quad \mu, \nu \in \mathcal{P}(\H),
	\end{equation} 
	where $\Gamma(\mu, \nu)$ denotes the set of couplings between $\mu$ and $\nu$.

%

	\begin{theorem}[Weak Harris Theorem; Theorem 4.7 in \cite{hairer2011asymptotic}] Let $\mathcal{H}$ be a separable Hilbert space and let $P$ be a transition kernel for a discrete time Markov chain with state space $\mathcal{H}$ for which the following conditions are satisfied:
		\begin{enumerate}
			\item (Lyapunov functional) There exists a lower semi-continuous function $V : \mathcal{H} \rightarrow[0, \infty) $ such that 
			\begin{equation}
			PV(u):= \int_{\H} V(w) P(u,dw)  \leq l V(u) + K, \quad \forall u \in \H, 
			\label{Lyapunov}
			\end{equation}
			where $K>0$ and $0<l<1$ are some constants.
			\item ($d$-contraction) There exist a distance like function 
			$d : \H \times \H \rightarrow [0,1]$ and a constant $\varrho \in (0,1)$ such that, for all $u,w \in \H$ with $d(u,w) <1,$
			\[ d(u,w) \leq \varrho.  \]
			\item ($d$-smallness of level sets of $V$) For the distance like function $d$ above, the functional $V$ and the constant $K$ in \eqref{Lyapunov}, there exists $\vartheta \in (0,1)$ such that,
			for all $u, w $ with $V(u), V(w) \leq 4 K,$   
			\[ d(u,w) \leq  \vartheta.  \]	
		\end{enumerate}
		Then, the Markov chain $P$ has a $\tilde{d}$-Wasserstein spectral gap where $\tilde{d}$ is the distance like function
		\[ \tilde{d}(u,w)= \sqrt{ d(u,w) ( 1 + V(u) + V(w)) }, \quad u,w \in \H.\]
		More precisely, there exist $\lambda>0$ and $C>0$ such that 
		\[ \tilde{d}(P^j \mu, P^j\nu) \leq C \exp(- \lambda j) \tilde{d}(\mu, \nu), \quad \forall \mu, \nu \in \mathcal{P}(\H), \quad \forall j \in \N.\]
		\label{WeakHarris}
	\end{theorem}
	
	\begin{remark}
		As remarked in \cite{hairer2011asymptotic}, we highlight that the second hypothesis is an assumption that holds for points $u,w$ with $d(u,w) <1$ and that nothing is being stated about points for which $d(u,w)=1$. The observation here is that even if one cannot deduce a Wasserstein spectral gap for the distance like function $d$, one can still obtain a Wasserstein spectral gap for the distance like function $\tilde{d}$.
	\end{remark}

	It is possible to quantify the constants $\lambda$ and $C$ in the conclusion of Theorem \ref{WeakHarris} in terms of the parameters $l, K,\varrho, \vartheta $. Here, however, we are simply interested in pointing out how changing the parameters in the assumptions affects the constants in the conclusions. In particular, it can be seen from the analysis in \cite{hairer2011asymptotic} that growth of any of the parameters $l,K, \varrho, \vartheta $ causes an increase in the constant $C$ and a decrease in the constant $\lambda$. In other words, enlarging any of the parameters $l,K,\varrho, \vartheta$ results in a worse spectral gap. This observation is relevant in order to obtain uniform spectral gaps for a sequence of Markov chains.  Namely, suppose that we have Markov kernels $\{P_n\}_{n \in \N}$ (with perhaps different state spaces) for which we can find distance like functions $\{ d_n \}_{n \in \N}$ and Lyupanov functionals $\{V_n \}_{n \in \N}$ satisfying the conditions in theorem \ref{WeakHarris} with constants  $\tilde{l}, \tilde{K}, \tilde{\varrho},  \tilde{ \vartheta}$ (independent of $n$). We can then deduce that the constants $\lambda>0$ and $C>0$ in the conclusion of the weak Harris theorem can be chosen independently of $n$. It is precisely this observation that is exploited in \cite{hairer2014spectral}

It is then important to highlight the main differences between our set-up and the one in \cite{hairer2014spectral}. First, the Markov kernels that we consider in this paper are not defined on the same state space and in particular the log-likelihoods $\Phi_n, \Phi$, although related, are different. Secondly, our discretization of the continuum prior $\pii$ is the prior $\pin$ supported on $L^2(\gamma_n)$ and not the discretization constructed by truncating the Karhunen Lo\`{e}ve expansion of the continuum prior.  These differences in the set-ups, however, do not prevent us from using the proof of Theorem 4.7 in \cite{hairer2014spectral} thanks to the following three observations.
	
	\nc
	
	\begin{enumerate}
		\item (Uniform control on local Lipschitz constants of log-likelihoods)

		\begin{lemma}
			\label{Lemma1}
			There exists a constant $L>0$ such that for every $r>0$ and $n \in \N$
			\[  \sup_{u_n , v_n \in \mathcal{B}^n_r } \frac{| \Phi_n(u_n;y) - \Phi_n(v_n;y) |}{\lVert u_n - v_n \rVert} \leq Lr, \quad   \sup_{u , v \in \mathcal{B}_r } \frac{| \Phi(u;y) - \Phi(v;y) |}{\lVert u - v \rVert}\leq Lr,\]
			where in the above $\mathcal{B}^n_r$ ($\mathcal{B}_r$) denotes the ball in $L^2(\gamma_n)$ ($ L^2(\gamma)$) centered at the origin and with radius $r$.
		\end{lemma}
		
		\begin{proof}
						Recall that 
			\[ \Phi_n(u_n;y) = \phi^y ( \G_n(u_n)), \quad u_n \in L^2(\gamma_n), \]
			and so, thanks to Assumptions \ref{Assumptionphi} on $\phi^y$, we get
			\begin{align*}
			\lvert \Phi_n(u_n;y) - \Phi_n(v_n;y) \rvert & \leq   \lvert \phi^y (\G_n(u_n) ) - \phi^y (\G_n(v_n)) \rvert 
			\\& \leq C_1 \max \{ \lvert \G_n(u_n)\rvert, \lvert \G_n(v_n)\rvert ,1 \} \lvert \G_n(u_n) - \G_n(v_n) \rvert.
			\end{align*}
			Now, recall that the vector $\mathcal{G}_n(u_n) - \mathcal{G}_{n}(v_n) \in \R^p$ has coordinates
			\[ [\mathcal{G}_n(u_n) - \mathcal{G}_{n}(v_n) ]_i =  \frac{1}{\gamma_n( B_\delta(\x_i))}    \langle  \mathds{1}_{B_\delta(\x_i)}, \F_n(u_n) - \F_n(v_n)  \rangle_{L^2(\gamma_n)}, \quad i=1, \dots, p.  \]
			From the Cauchy-Schwartz inequality it follows that
			\begin{align*}
			\lvert  [\mathcal{G}_n(u_n) - \mathcal{G}_{n}(v_n) ]_i  \rvert  & \leq  \frac{1}{(\gamma_n( B_\delta(\x_i) ))^{1/2}}\lVert  \F_n(u_n) - \F_n(v_n)  \rVert_{L^2(\gamma_n)}   
			\\& \leq  \frac{1}{(\gamma_n( B_\delta(\x_i) ))^{1/2}}\lVert  u_n- v_n  \rVert_{L^2(\gamma_n)}, 
			\end{align*}
			where in the last line we have used the fact that $\mathcal{F}_n$ is a linear map as well as the fact that it is a contraction. Since 
			\begin{equation}
			\gamma_n( B_\delta(\x_i) ) \rightarrow \gamma(B_\delta(\x_i) ), \quad \text{ as } n \rightarrow \infty,
			\label{AppAux1}
			\end{equation}
			it follows that 
			\[  \lvert  \G_n(u_n) - \G_n (v_n)    \rvert  \leq C_2 \lVert u_n - v_n   \rVert_{L^2(\gamma)}, \]  
			where $C_2$ is independent of $u_n, v_n \in L^2(\gamma_n)$ or $n \in \N$. Therefore, there exists a constant $C_3$ (independent of $u_n, v_n \in L^2(\gamma_n)$ or $n \in \N$) such that
			\[  \lvert \Phi_n(u_n;y) - \Phi_n(v_n;y) \rvert \leq C_3 \max \{ \lVert  u_n\rVert_{L^2(\gamma_n) }  , \lVert v_n \rVert_{L^2(\gamma_n)}, 1 \}  \lVert u_n - v_n \rVert_{L^2(\gamma_n)}.\]
			Naturally the same analysis holds for $\Phi$ and this finishes the proof. 
		\end{proof}

		\begin{remark}
			The conclusions in the previous lemma hold for non-linear forward maps $\F_n$, $\F$ that are (uniformly in $n$) Lipschitz and have (uniformly in $n$) linear growth.
		\end{remark}

		\item (Dominating limiting measure) We make use of a ``limiting measure'' that dominates the measures $\pin$ in the sense described below. Notice that we cannot use the continuum prior $\pii$, but a slight modification of it will suffice. 
		
		\begin{lemma}
			\label{Lemma2}
			There exists a large enough $\rho>0$, such that the Gaussian measure 
			\[ \pii^\rho:= N\bigl(0, (1+\rho)^2 (\alpha I - \Delta_\M)^{-s} \bigr), \]
			satisfies
			\[      \int_{L^2(\gamma_n)} g(  \lVert  u_n \rVert_{L^2(\gamma_n)} )d \pin(u_n)  \leq      \int_{L^2(\gamma)} g(  \lVert  u \rVert_{L^2(\gamma)} )d \pii^\rho(u), \]
			for every $n \in \N$ and every increasing function $g :[0,\infty) \rightarrow \R$. In particular, for every $r>0$ and every $n \in \N,$
			\[\pin \left(  L^2(\gamma_n) \setminus \mathcal{B}_r^n \right) \leq \pii^\rho \left( L^2(\gamma) \setminus \mathcal{B}_r  \right).\]
			
		\end{lemma}
		\begin{proof}
			Thanks to inequality \eqref{th:pcn}, we can find $\rho>0$ such that, for every $n \in \N,$
			\[  \frac{1}{(\alpha  + \lambda_i^n)^s} \leq \frac{1+\rho}{(\alpha + \lambda_i)^s} , \quad \forall i=1, \dots, k_n.   \]
			Using the Karhunen Lo\`{e}ve expansion to represent random variables with laws $\pin$ and $\pii^\rho$ we can easily deduce the inequality for the measures of complements of balls (last inequality).  The inequality for a general increasing function $g$ follows from a standard approximation with increasing step functions. 
		\end{proof}

		\item (Uniform lower bound for acceptance probability) The next lemma provides uniform control on the acceptance probability of the pCN algorithm when a proposal lies within a fixed distance of a contracted version of the current state of the chain. More precisely:

		\begin{lemma}
			\label{Lemma3}
			Let $a(u,v)$ be the acceptance probability in Algorithm 1 for continuum pCN and $a_n(u_n, v_n)$ the acceptance probability in Algorithm 2 for graph pCN. Fix an arbitrary $r>0$. Then, there exists  $c\in \R$ such that
			\[  \inf_{w_n \in \mathcal{B}^n_r(\sqrt{1-\beta^2} v_n )} a_n(v_n , w_n) \geq \exp(c) >0 , \quad \inf_{w \in \mathcal{B}_r(\sqrt{1-\beta^2} v )} a(v , w) \geq  \exp(c) >0  \]
			for arbitrary $v_n \in L^2(\gamma_n)$, $v \in L^2(\gamma)$ and $n \in \N$. 
		\end{lemma}

		\begin{proof}
			First of all notice that  
			\[ \lVert \G_n \rVert \leq \lVert \O_n \rVert \lVert \F_n \rVert \leq \lVert \O_n\rVert, \]
			where in the last inequality we have used that $\F_n$ is a contraction. Thanks to \eqref{AppAux1} it follows that
			\[ \lVert \O_n \rVert \rightarrow \lVert \O\rVert, \quad \text{ as } n \rightarrow \infty,\]
			and in particular we can find a constant $\tilde{K}$ (independent of $n$) such that
			\[  \lVert \mathcal{G}_n \rVert \leq \tilde{K}.\]
			
			Let $v_n, w_n \in L^2(\gamma_n)$ be such that $w_n \in \B_r^n(\sqrt{1-\beta^2}v_n )$. Then,
			
			\begin{align*}
			 \lvert \G_n (w_n) - \sqrt{1-\beta^2} \G_n(v_n) \rvert &= \lvert \G_n (w_n- \sqrt{1-\beta^2} v_n)  \rvert  \\
			 &\leq \lVert \mathcal{G}_n \rVert \lVert w_n - \sqrt{1-\beta^2}v_n \rVert_{L^2(\gamma_n)} \\
			  &\leq \tilde{K} r =: K.
			\end{align*}

			From Assumptions \ref{Assumptionphi} we deduce that 
			\[ \Phi_n(v_n;y) - \Phi_n(w_n;y) = \phi^y (\G_n(v_n) ) - \phi^y (\G_n(w_n)) \geq c, \]
			for a $c$ that is independent of $n$. Hence, 
			\[  \inf_{w_n \in \mathcal{B}^n_r(\sqrt{1-\beta^2} v_n )} a(v_n , w_n) \geq \exp(c) >0. \]
			Naturally the same analysis holds for $\Phi$ and this finishes the proof. 
		\end{proof}
		\nc

		\begin{remark}
			The same conclusions in the previous lemma hold for non-linear forward maps $\F_n$, $\F$ that are (uniformly in $n$) Lipschitz,  have (uniformly in $n$) linear growth, and are positively homogeneous of degree one.	
		\end{remark}

	\end{enumerate}

	\begin{proof}[Proof of Theorem \ref{th:pcn}]
	Lemmas \ref{Lemma1}, \ref{Lemma2} and \ref{Lemma3} allow us to follow the analysis in \cite{hairer2014spectral} (where in our case we use $\pii^\rho$ from Lemma \ref{Lemma2}) and check that the conditions of the weak Harris theorem (with distance like functional $d$ and Lyapunov functional $V$ as in the statement of our theorem) are satisfied with constants $l,K,\varrho,  \vartheta$ that are independent of the discretization.
	
		
	\end{proof}

	\begin{proof}[Proof of Corollary \ref{corGAPS}] By Proposition 2.8 and Lemma 2.9 in \cite{hairer2014spectral}, and the reversibility of the Markov kernel of the pCN algorithm, it is enough to check that the space 
	\[ Lip(\tilde{d} ) \cap L^\infty(\H ; \mu), \]
is dense in $L^2(\H ; \mu)$. Here $\tilde{d}$ denotes the distance-like function from Theorem \ref{th:pcn} and $\mu$ stands for the invariant measure of the Markov chain (in this case the posterior distribution). In the finite dimensional case (i.e. $\H = L^2(\gamma_n)$) this is a simple consequence of a standard mollification argument. More precisely, it follows from the following observations:
\begin{enumerate}
\item For every $R>0$, $\lVert \cdot \rVert$-Lispchitz functions on $\mathcal{B}^n_R$ are also $\tilde{d}$-Lipschitz on $\mathcal{B}_R^n$.
\item $\lVert \cdot \rVert$-Lispchitz functions on $\mathcal{B}_R^n$ are dense in $L^2(\mathcal{B}^n_R ; \mu  )$ (by mollification).
\item $f \in L^2(\H; \mu)$ can be approximated with  $ \{ f_k \}_{k \in \N} $, where
\[  f_k(u) := \eta_k( \lVert u \rVert) \min  \{ \max \{  f(u) , -k   \}, k \}, \quad u \in \H, \]
where $\eta_k: [0, \infty) \rightarrow [0, 1]$ is a smooth cut-off function which  satisfies $\eta_k(r)=1$ if $r < k$ and $\eta_k(r)=0$ if $r>2k$.
\end{enumerate}

For the infinite dimensional case it is enough to reduce the problem to the finite dimensional case. This reduction is achieved as follows. Without the loss of generality an arbitrary element $u \in \H$ can be written as $u=(u_1,u_2, \dots)$ and for every $k \in \N$ we may consider the projection:
\[   \Pi_k^c : u \in \H \mapsto (u_{k+1}, u_{k+2}, \dots), \]
 and the measure  $\mu_k^c:= \Pi^c_{k\sharp} \mu$. For an arbitrary $f\in L^2(\H; \mu)$, we can then define the sequence $\{  f_k \}_{k \in \N} \subseteq L^2(\H; \mu)$ defined by
 \[   f_k( u) :=  \int f(u_1, \dots, u_k, v_{k+1}, v_{k+2}, \dots) d \mu_k^c  (v_{k+1}, v_{k+2}, \dots),  \quad u \in \H  \] 
Notice that for each $k$ the function $f_k$ depends only on the first  $k$ coordinates of $u$ and so we can apply the result for the finite dimensional case to approximate $f_k$ with functions in $Lip(\tilde{d} ) \cap L^\infty(\H ; \mu)$. From the straightforward fact that  $f_k \rightarrow_{L^2(\H; \mu)} f$, the approximation of functions in $L^2(\H; \mu)$ with functions in $Lip(\tilde{d} ) \cap L^\infty(\H ; \mu)$ now follows.

	\end{proof}

	\section{Verification of Hypotheses for Gaussian and Probit Noise Models}
	\label{App3}

	\subsection{Gaussian}  Let us show that the Gaussian model satisfies Assumption \ref{Assumptionphi}.
	\begin{enumerate}
		\item

		Let $K>0$ and let $\tau>0$ be such that $(1-\tau) > (1+\tau)(1-\beta^2)$. For such $\tau$ choose  $R=R_\tau >0$ large enough so that if $u \in \R^p$ satisfies $\lvert u \rvert \geq R$ then 
		\[ (1-\tau ) \lvert u \rvert^2  \leq  \lvert u -y \rvert^2 \leq (1+\tau) \lvert u \rvert^2. \]
		
		Let  $v, w\in \R^p$ be such that $\lvert w - \sqrt{1-\beta^2} v \rvert \leq K$.		
		If $|w | \leq R+K$, then
		\[ \lvert v - y \rvert^2 - \lvert w - y \rvert^2 \geq 0 -  2\lvert y \rvert^2 - 2(R+K)^2 .\]
		On the other hand, if $| w | \leq R+K$, we see that
		\[ R+K \leq \lvert w \rvert \leq \sqrt{1-\beta^2} \lvert v \rvert  + K, \]
		and it follows that
		\begin{align*}
		\lvert v - y \rvert^2 - \lvert w - y \rvert^2 & \geq (1-\tau) \lvert v \rvert^2 -(1+\tau) \lvert w \rvert^2  \\
		&\geq ((1-\tau) -(1+\tau)(1-\beta^2) ) \lvert v \rvert^2 -2(1+\tau)\sqrt{1-\beta^2}K \lvert v \rvert -(1+\tau)K^2\\
		& \geq C_1,
		\end{align*}
		for some real number $C_1$. 
		
		From the above analysis we deduce that for every $v, w$ with $\lvert w -\sqrt{1-\beta^2}\, v \rvert \leq K$,  
		\[ \phi^y(v) - \phi^y( w ) \geq c,   \]	
		for some $c \in \R$.
		\item The second assumption is easily seen to be satisfied by the Gaussian model. 
	\end{enumerate}

	\subsection{Probit} Let us show that the probit model satisfies Assumption \ref{Assumptionphi}.
	
	\begin{enumerate}
		\item  Let $K>0$ and consider $v, w \in \R^p$ such that $\lvert w - \sqrt{1-\beta^2} v \rvert \leq K$. Then, 
		\begin{equation}
		\lvert y_i w_i - \sqrt{1-\beta^2} y_i v_i\rvert  =  \lvert w_i - \sqrt{1-\beta^2}  v_i\rvert \leq \lvert v - \sqrt{1-\beta^2} \, w \rvert \leq K   , \quad i=1, \dots, p,
		\label{AuxC1}
		\end{equation}
		where the first equality follows from the fact that $y_i \in \{ -1, 1\}.$ In particular,
		\begin{equation} 
		\lvert  y_i w_i  \rvert \leq K + \sqrt{1-\beta^2}\lvert y_i v_i \rvert, \quad i=1, \dots, p.  
		\label{AuxC2}
		\end{equation}
		
		Notice that the function $t \in \R \mapsto -\log( \Psi(t))$ is decreasing. Hence, if $y_i w_i> -(1/(1-\sqrt{1-\beta^2}) +1)K$ we see that 
		\[ - \log(\Psi( y_i v_i    ))  -  (- \log(\Psi( y_i v_i    )) ) \geq 0+ \log(\Psi(- (1/(1-\sqrt{1-\beta^2}) +1)K )). \] 
		On the other hand, if $y_i w_i < -(1/(1-\sqrt{1-\beta^2}) +1)K$ we deduce from \eqref{AuxC1} that $y_i v_i<- K/(1-\sqrt{1-\beta^2}) < 0 $ and from \eqref{AuxC2} we deduce 
		\[  y_i v_i  \leq \sqrt{1-\beta^2} y_i v_i  -K  \leq    y_i w_i, \]
		from where it follows that 
		\[ - \log(\Psi( y_i v_i    ))  -  (- \log(\Psi( y_i v_i    )) ) \geq 0. \]
		
		From the above analysis we deduce that, for every $v, w$ with $\lvert w-\sqrt{1-\beta^2} v \rvert \leq K$, 
		\[ \phi^y(v) - \phi^y( w) \geq c:= p  \log(\Psi(- (1/(1-\sqrt{1-\beta^2}) +1)K )).   \]
		
		\item Let us now check that the probit model satisfies the second assumption on $\phi^y$. Since the function
		\[ g: t \in \R \mapsto - \log(\Psi(t)) \]
		is decreasing, convex, and converges to zero as $t \rightarrow \infty$, the first assumption on $\phi^y$ will hold if we can show that
		\[  \limsup_{t \rightarrow -\infty} \frac{\lvert  g'(t)\rvert}{\lvert t \rvert} < \infty.\]
		This however follows from the fact that
		\[ g'(t) = \frac{- e^{-t^2/2}}{\int_{-\infty}^{t} e^{-r^2/2} dr } ,  \]
		and the well known fact that
		\[ \frac{e^{-t^2/2}}{2\lvert t\rvert} \leq \int_{-\infty}^t e^{-r^2/2} dr ,   \]
		for all negative enough $t$.
		
	\end{enumerate}

\end{document}